\documentclass[10pt]{article}

\usepackage[preprint]{tmlr}

\usepackage{amsmath,amssymb,amsfonts,amsthm}
\usepackage{graphicx}
\usepackage{algorithm}
\usepackage{algpseudocode}
\usepackage{array}
\usepackage{booktabs}
\usepackage{bm}
\usepackage{url}
\usepackage[hidelinks]{hyperref}
\usepackage{xcolor}
\usepackage{mathrsfs}

\newtheorem{theorem}{Theorem}
\newtheorem{definition}{Definition}
\newtheorem{assumption}{Assumption}
\newtheorem{lemma}{Lemma}
\newtheorem{proposition}{Proposition}
\newtheorem{corollary}{Corollary}
\newtheorem{remark}{Remark}

\newcommand{\R}{\mathbb{R}}
\newcommand{\Cset}{\mathbb{C}}
\newcommand{\Mx}{\mathcal{M}_x}
\newcommand{\Ms}{\mathcal{M}_s}
\newcommand{\Mz}{\mathcal{M}_z}

\newcommand{\Loss}{\mathcal{L}}
\newcommand{\Ybus}{\mathbf{Y}_{\text{bus}}}
\newcommand{\GSO}{\mathbf{L}}
\newcommand{\zx}{z_x}
\newcommand{\zs}{z_s}
\newcommand{\Px}{\mathscr{P}_x}
\newcommand{\Ps}{\mathscr{P}_s}

\begin{document}

\title{Graph Transfer Learning via Shared Latent Geometry:\\
Theory and Applications}

\author{%
  \name Tong Wu\thanks{Part of this work was conducted while visiting Cornell University.}
  \email{tong.wu@ucf.edu} \\
  \addr University of Central Florida, USA
  \AND
  \name Andrew Campbell \email{ac2458@cornell.edu} \\
  \addr Cornell University, USA
  \AND
  \name Anna Scaglione \email{as337@cornell.edu} \\
  \addr Cornell University, USA
}
\maketitle

\begin{abstract}
Inference and control in engineered physical systems pay a heavy physics cost at every deployment: state estimators, solvers for inverse problems, model-predictive controllers, schedulers, and observer designs are often not closed form solutions and hence need to re-solve a numerical optimization per instance, with the system operator re-supplied each time. Physics-informed machine-learning approaches move this cost to training, but use a single encoder pathway whose latent geometry de-learns under fine-tuning on a new instance and admits no quantitative transfer guarantee. We propose an asymmetric two-pathway architecture that resolves both issues. A teacher encoder consumes a privileged dense state from a high-fidelity simulator and represents the system through operator-polynomial features whose coefficients are stable under spectral perturbation; a student encoder is trained to reach the same latent geometry from sparse field data and operator descriptors. At deployment the teacher is discarded, and the frozen student runs in a single forward pass accompanied by a quantitative transfer certificate. The asymmetric design connects to learning using privileged information, knowledge distillation, and cross-modal distillation, but targets cross-instance transfer rather than fixed-instance prediction: topology and operator may change between training and deployment, while the latent task does not. We establish sufficient and near-necessary conditions for transfer in terms of Wasserstein proximity between latent laws, yielding a zero-shot per-instance error bound, and develop a finite-sample certification protocol with certificate-guided active expansion when coverage is incomplete. The framework applies wherever a system admits a well-defined operator with a reportable spectrum (Laplacians, admittance matrices, discretized Green's functions, channel matrices, state-transition Jacobians) and we instantiate it on graph-structured systems. On power-system state estimation across distribution networks, the architecture achieves zero-shot transfer to 100 unseen topologies of varying size, with a 95\% certificate pass rate, accuracy competitive with topology-aware Newton--Raphson state estimators that require exact topology, and sub-millisecond inference independent of network size. These results suggest that the asymmetric pathway plus operator-anchored latent geometry provides a principled foundation for certified zero-shot inference and control across heterogeneous physical systems.
\end{abstract}
\section{Introduction}
\label{sec:intro}

Inference and control in engineered physical systems (state estimation, observer design, channel equalization, model-predictive control) pay a heavy physics cost at every deployment. A Newton--Raphson power-flow solve, an extended Kalman update, a model-predictive control program, or a channel-matrix inversion is correct by construction, but it must be re-solved for every new instance, with the system operator (an admittance matrix, a state-transition Jacobian, a channel matrix, a discretized Green's function) re-supplied each time. The computational burden lives at run time, where it is most expensive and most fragile to model changes: the physics is \emph{interpreted}, not \emph{compiled}. The conventional machine-learning response (train a single encoder end-to-end on representative instances and fine-tune when the environment changes) replaces one problem with another. A monolithic encoder, transformer or otherwise, shapes its latent geometry around whatever distribution it has most recently seen; fine-tuning on a new operating point de-learns the geometry that supported the previous ones, and zero-shot deployment on a genuinely new instance has no structural reason to succeed. This paper develops an architecture that resolves both issues at once. The physics is paid once, at training, by a teacher pathway that consumes a privileged dense state from a high-fidelity simulator and represents the system through operator-polynomial features whose coefficients are stable under spectral perturbation. A second pathway (the student) is trained to land in the same latent geometry from the sparse, deployment-time observation alone. The teacher locks the latent geometry to the physics \emph{before} the student is asked to reach it, so the geometry persists in the student's weights even when the student is later exposed to new instances; deployment is a zero-shot forward pass, accompanied by a quantitative certificate that determines whether the new instance falls within the architecture's transfer envelope. The asymmetry between the two training pathways (privileged at training, deployable at test, anchored to the operator throughout) is what makes both the amortization and the zero-shot transfer sound.

The template applies wherever a system admits a well-defined operator with a reportable spectrum (a Laplacian, an admittance matrix, a discretized Green's function, a channel matrix, a state-transition Jacobian, or any analogous object whose eigenstructure carries the physics of the instance. The present paper develops the theory at this generality and validates it on graph-structured systems, the case in which the operator is a graph shift operator (GSO) and the deployment observation is a sparse signal on the nodes. Graph signal processing \citep{shuman2013emerging,ortega2018graph} is a particularly clean instantiation: in engineered and networked physical systems) power grids, sensor and communication networks, water and gas distribution networks, transportation networks, structural-mechanical systems, the graph itself encodes physically meaningful structure (adjacency as connectivity, edge weights as couplings, the GSO as dynamics), and topology variation poses a transfer problem with three compounding difficulties. First, the spectrum and message-passing basis of the governing operator change between instances, so source and target features are not directly comparable. Second, the node count $N$ may change, so the dimensionality of inputs and outputs changes too. Third, the semantic meaning of individual nodes need not align across topologies: a re-routed power line or a relocated sensor changes what a node \emph{is}. Classical domain adaptation \citep{pan2010survey,ganin2016domain} requires a shared feature space; graph neural networks \citep{kipf2017semi,defferrard2016convolutional} permit variable $N$ but still expect one parameter family to generalize across topologies; both break when graph structure itself carries signal semantics, as it does in power systems when a line outage simultaneously changes the feature distribution and the physical meaning of the outputs. A general framework must therefore operate at the level of graph-induced signal \emph{manifolds}, not topology-specific coordinates.
\begin{figure}[t]
  \centering
  \includegraphics[width=\linewidth]{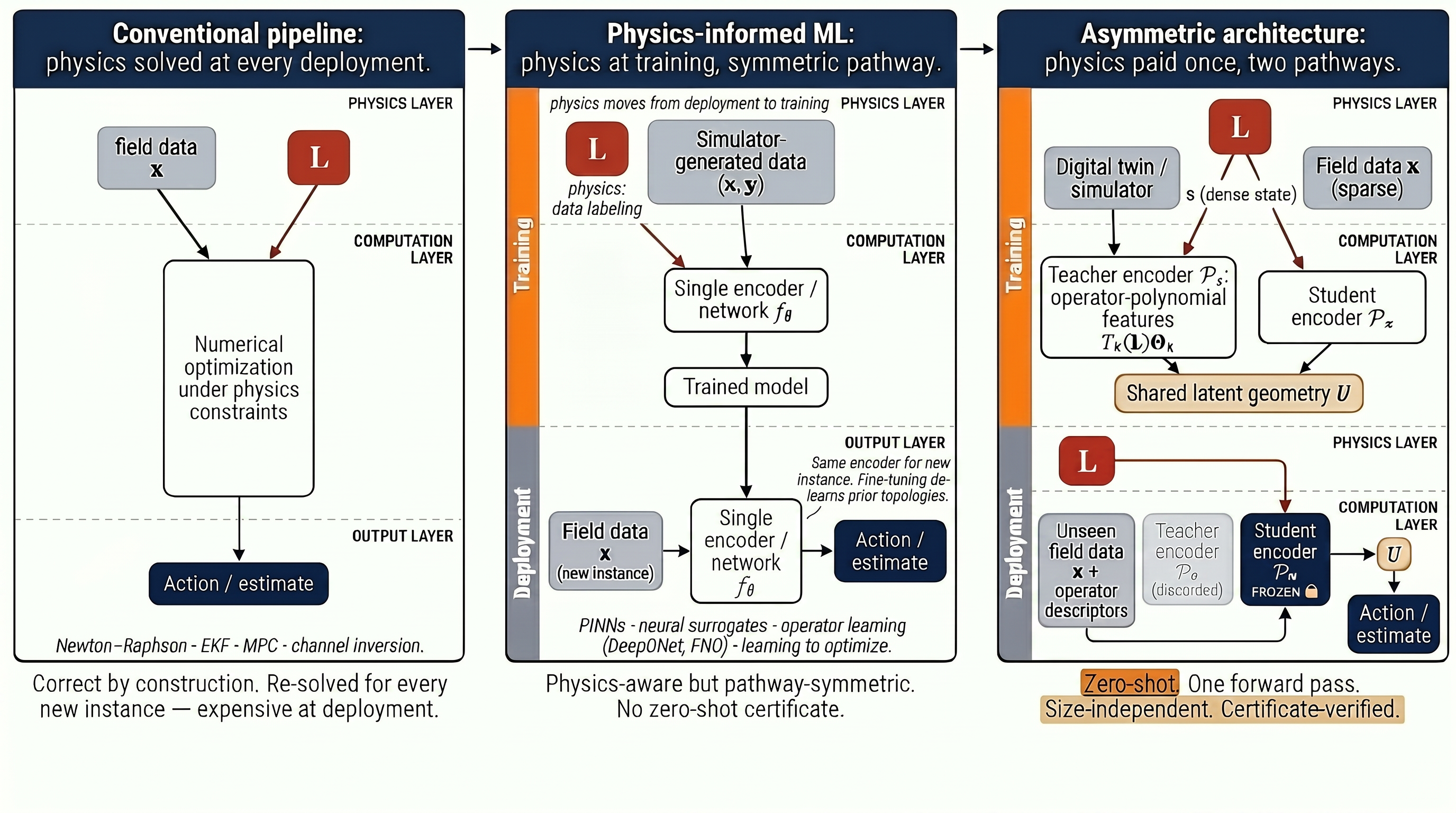}
  \caption{Three regimes of physics in inference and control pipelines.
  \textbf{Left:} conventional pipelines (Newton--Raphson, EKF, MPC, channel
  inversion) re-solve a physics-constrained optimization at every deployment
  instance, the system operator $\GSO$ must be re-supplied each time.
  \textbf{Middle:} physics-informed machine-learning approaches (PINNs, neural
  surrogates, operator learning, learning-to-optimize) move the physics cost to
  training, but use a single encoder for both training and deployment;
  fine-tuning on a new instance de-learns the geometry that supported earlier
  ones, and no quantitative transfer certificate is available.
  \textbf{Right:} the proposed asymmetric architecture pays the physics cost
  once, at training, via a teacher pathway $\Ps$ whose features are polynomials
  in $\GSO$ and a student pathway $\Px$ trained to reach the same latent
  geometry $U$ from sparse field data. The operator $\GSO$ is the through-line:
  it enters both encoders at training and the student at deployment. Only the
  student is retained at test time, and a transfer certificate determines
  whether a new instance falls inside the architecture's transfer envelope.}
  \label{fig:compilation_metaphor}
\end{figure}

\subsection{The Core Question}
We ask: \emph{under what geometric conditions can a learned latent space
support valid transfer across instances of an operator family of different
size, topology, and signal type?} Our answer is that transfer is possible
only when the tasks defined on different instances admit a \emph{shared
latent coordinate system}: the system operator may change the coordinate
chart (and with it the dimensionality and the message-passing basis), but
it must not change the latent task realized by the decoder.

The intuition is that a transferable latent space must preserve the local
semantic organization of each instance while placing different instances onto
a common connected latent set. Nearby signals should remain nearby, distinct
states should remain distinguishable, and sparse observations should land
near the latent code of the corresponding dense privileged state. When these
conditions hold simultaneously, instances become comparable in latent
coordinates, which is precisely what enables zero-shot transfer to instances
unseen during training.

\subsection{Asymmetric Training, Symmetric Geometry}
A distinctive feature of our formulation is the \emph{asymmetry} between
the training and deployment pathways. During training we observe paired
data $(x, s)$ where $x$ denotes the sparse or partial observation available
at deployment, and $s$ denotes a dense privileged signal available only
during training, for example, full physical states from a simulator or
digital twin, future trajectories, or privileged measurements that will
not be available at run time. Two encoders process these signals
independently: $\Ps$ is a teacher encoder operating on the dense
privileged signal $s$, and $\Px$ is a structure-aware encoder operating
on the sparse observation together with a description of the system
operator. Both paths project into the same $d$-dimensional latent
manifold and share a single adaptive decoder. A geometric consistency
loss pulls the two latent codes together, so that $\Ps$ shapes a latent
geometry that $\Px$ cannot see directly, and $\Px$ learns to reach the
same region from sparse input alone. At test time only $\Px$ is used;
$\Ps$ and the privileged signal $s$ are discarded, and the geometry
persists implicitly through the weights of $\Px$.

The teacher--student structure with asymmetric information has antecedents
in Vapnik and Vashist's \emph{Learning Using Privileged Information}
\citep{vapnik2009new}, Hinton's \emph{knowledge distillation}
\citep{hinton2015distilling}, the intermediate-representation matching of
FitNets \citep{romero2015fitnets}, and \emph{cross-modal distillation}
\citep{gupta2016cross}; the first two are unified under the generalized
distillation framework of \citet{lopezpaz2016unifying}. Our setup borrows
from all of these: the privileged signal is informationally richer than
the deployment signal (LUPI), the student is trained to match a teacher
representation rather than ground-truth labels alone (distillation), the
matching happens in an intermediate latent space rather than at the
output layer (FitNets), and the teacher and student operate on different
modalities (cross-modal). What distinguishes the present work is its
target. Classical distillation, LUPI, and cross-modal distillation train
a student to perform a fixed prediction task on the \emph{same} system
seen at training; we train a student whose latent representation supports
zero-shot transfer to new instances, different operator, different
dimensionality, different semantic node identities. The asymmetric
pathway is what shapes the latent geometry; the operator-polynomial
features in the teacher and the operator descriptors in the student are
what make that geometry transferable across instances.

\subsection{Contributions}
This paper makes the following contributions:
\begin{enumerate}
  \item \textbf{Asymmetric dual-path architecture.} A general architecture in
  which a teacher encoder $\Ps$ shapes a shared latent geometry using a
  privileged dense signal during training, while a student encoder $\Px$
  over sparse observations and operator descriptors serves as the sole
  deployment pathway. A shared size-adaptive decoder reconstructs node- or
  graph-level targets at any size. The teacher is discarded at deployment;
  only $\Px$ and the decoder are retained.

  \item \textbf{Transfer-compatibility criterion and zero-shot error bound.}
  A geometric criterion under which cross-instance transfer is well-posed:
  an operator family is \emph{transfer-compatible} when its privileged
  signal manifolds embed into a shared latent set on which a single
  size-adaptive decoder realizes the task. Under this condition we derive
  a per-instance error bound on the student's zero-shot prediction in
  terms of teacher reconstruction fidelity and a latent-alignment term,
  neither of which requires the teacher or the privileged signal at
  deployment.

  \item \textbf{Latent-law transfer guarantees, sufficient and near-necessary.}
  Sufficient and near-necessary conditions on the latent laws induced by
  the two pathways. The sufficient direction shows that Wasserstein
  proximity between the deployment law and a reference law, combined with
  a Lipschitz loss-composition condition, implies bounded task risk on
  unseen instances. The necessary direction shows that, under a natural
  identifiability condition, successful transfer forces the same law
  proximity, yielding a near-equivalence between latent-law
  compatibility and zero-shot transfer with bounded task risk.

  \item \textbf{Verification protocol and certificate-guided expansion.}
  A finite-sample protocol that issues a quantitative
  pass/fail certificate per instance, combining reconstruction fidelity,
  latent-law Wasserstein proximity, and cross-modal alignment.
  The protocol is not merely diagnostic: the certificate
  identifies uncovered instances in a candidate pool, guiding targeted
  fine-tuning that expands latent coverage without full retraining.
\end{enumerate}

\section{Related Work}
\label{sec:related}

\subsection{Teacher--Student Learning, Privileged Information, and Distillation}
The asymmetric dual-path architecture proposed here is rooted in a family
of learning paradigms in which a richer training-time signal shapes a
deployment-time predictor. Vapnik and Vashist's \emph{Learning Using
Privileged Information} (LUPI) framework \citep{vapnik2009new} formalized
the asymmetry: a privileged signal $s$ is available during training but
not at deployment, and the student must learn a predictor on the
deployable input alone while exploiting $s$ during training.
\emph{Knowledge distillation} \citep{hinton2015distilling} established
the teacher--student form of this idea, in which a stronger teacher's
outputs supervise a smaller or sparser-input student; \emph{FitNets}
\citep{romero2015fitnets} extended distillation from output-level to
intermediate-representation matching, anticipating the latent-alignment
role that the consistency loss plays in our framework.
\emph{Cross-modal distillation} \citep{gupta2016cross} brought asymmetric
modalities into the teacher--student setting, training a student on a
target modality to match a teacher operating on a different source
modality on the same scene. \citet{lopezpaz2016unifying} unified LUPI
and distillation under generalized distillation, making the equivalence
of privileged-information and teacher-output supervision explicit. On
graph-structured data, both paradigms have been applied through graph
knowledge distillation \citep{yang2020distilling,yan2020tinygnn}, where
a heavy GNN teacher trains a lightweight student on the same graph.

Our framework borrows the training-time information asymmetry from LUPI,
the latent-representation supervision from FitNets, and the modality
asymmetry from cross-modal distillation. What distinguishes it is the
\emph{target}. In all of these prior paradigms, the deployment system is
fixed: the student learns a predictor for the same scene, the same
graph, the same input domain seen at training. We instead target
cross-instance transfer: the operator $\GSO$ at deployment can be a
system never observed at training. The teacher-shaped latent geometry,
anchored by operator-polynomial features that are stable under spectral
perturbation, supports a single decoder family across instances of
varying size, topology, and operator structure. The distinguishing
technical commitment is the operator-anchoring of the teacher's feature
extractors, not the asymmetric pathway in isolation.

\subsection{Physics-Informed Machine Learning and Surrogate Models}
A second tradition relevant to our framework treats physics as a
supervisory or architectural prior. Physics-informed neural networks
\citep{raissi2019physics} embed governing equations directly into the
loss function as residual penalties, enforcing physical consistency on a
single end-to-end network without simulator data. Neural-surrogate and
operator-learning approaches instead use a high-fidelity simulator as
the source of training-time targets and train a single network to
approximate the solution map: in power systems, neural surrogates for
power flow and state estimation \citep{donon2020neural,owerko2020optimal}
take this route. The neural-operator literature generalizes the surrogate
idea to families of parametric PDEs, training a network that maps inputs
to solutions across the family, DeepONet \citep{lu2021deeponet} uses
a branch--trunk architecture grounded in the universal approximation
theorem for operators, and the Fourier Neural Operator
\citep{li2021fourier} parameterizes the operator in the spectral domain.
Learning-to-optimize and unrolled-solver methods \citep{zhang2019real}
replace iterative solvers with neural networks that mimic the solver's
update rule, accelerating per-instance inference while remaining within
the per-instance-solve paradigm.

These approaches share an architectural feature that distinguishes them
from ours: a single encoder pathway processes the same input modality
at training and at deployment, with physics entering through the loss
function, through the choice of training data, or through an
operator-aware inductive bias. There is no privileged training-time
pathway and no separate deployment-time pathway; the same network is
asked to do both jobs. Consequently, when a single-encoder model is
fine-tuned on a new operating point, the latent geometry that served
previous instances is overwritten, and the resulting model carries no
quantitative certificate that it remains in-distribution. The
asymmetric pathway makes both properties achievable simultaneously:
the teacher carries the physics into a latent geometry whose structure
is preserved across instances, while the student provides a
deployment-time forward map whose validity on a new instance can be
certified from the deployment observation alone.

\subsection{Domain Adaptation and Distribution Alignment}
A large body of work on heterogeneous transfer and graph domain
adaptation frames transfer as a distribution alignment problem
\citep{dai2022graph,wu2022handling,you2023graph}. These approaches
minimize structural discrepancy, align node or graph embeddings across
domains, or learn domain-invariant representations through adversarial
objectives \citep{ganin2016domain,shen2020adversarial}. The shared
classifier or shared feature space is typically assumed to exist; the
goal is to reduce the distributional gap so that a fixed hypothesis
transfers. Our framework differs in two respects. First, we do not
assume a pre-existing shared feature space: the latent support is
learned jointly by a privileged teacher pathway and a deployable
student pathway, and transfer compatibility is a property of the
learned geometry rather than an input assumption. Second, the proof
axis is not feature alignment per se but \emph{latent-law alignment}:
the teacher encoder defines a reference task law $\mu_\star$ in the
latent space, the student is trained to track it, and Wasserstein
proximity between the two induced laws is what enters the transfer
bound directly. Distribution-alignment methods typically do not
formalize this teacher-student latent-law structure, and accordingly
do not produce a necessary condition for transfer.

\subsection{Graph Transfer, Graphon Theory, and Operator-Perturbation Stability}
Within the graph-specific literature, two complementary lines of work
give transfer guarantees grounded in properties of the graph operator
itself. The graphon framework \citep{ruiz2021graphon,ruiz2023transferability}
treats graphs of varying size as samples from a common limiting object
and derives transfer error bounds for graph filters and GNNs across
scales in terms of sampling error and filter Lipschitz constants
\citep{maskey2022generalization,ruiz2020graphon}. The proof axis is
graph-limit theory: transfer is justified when source and target graphs
are samples from the same graphon family, so that spectral discrepancy
vanishes as size grows. A second line of work studies how GNNs respond
to perturbations of the graph operator, changes in topology, or
variability in sampling \citep{gama2020stability,verma2019stability,zhou2021graph},
and uses operator-norm continuity to bound generalization error across
graphs \citep{liao2021pac,esser2021learning}. The central object in
both is \emph{operator stability}: transfer is controlled by how much
the shift operator changes between source and target. Our framework
operates on a different axis. Operator stability is neither assumed nor
derived here; the teacher pathway defines a complete task family
through its induced latent law, the student pathway enters the same
latent region through consistency training, and Wasserstein risk
stability converts latent-law proximity into task error regardless of
how dissimilar the underlying operators are. This makes the framework
applicable to settings where source and target operators are
structurally incomparable, for example, distribution networks of
different bus counts, line topologies, and admittance structures , 
where operator-perturbation bounds would be vacuous. The two routes are
complementary: graphon and spectral-stability arguments control
operator-level discrepancy across size when a common limiting object
exists, while latent-law arguments control semantic-level discrepancy
across topology, operator type, and signal domain simultaneously.

\subsection{Manifold Learning, Homeomorphism, and Topological Methods}
Topology-preserving latent maps have appeared in unsupervised domain
adaptation \citep{zhou2023homeomorphic}, topological autoencoders
\citep{moor2020topological, wu2026universal}, and constrained optimization
\citep{liang2024homeomorphic, wu2026geometric}. Persistent homology has been used to
regularize latent spaces and measure representation quality
\citep{hofer2019connectivity,carriere2020perslay}. These works motivate
the importance of local geometric regularity in learned representations.
We adopt homeomorphism as one ingredient of the transfer-compatibility
criterion and persistent homology as one diagnostic in the verification
protocol, but neither is the core proof axis. The foundational claim of
this paper is that transfer reduces to a latent-law question: topology
and operator variation are absorbed as coordinate changes on a shared
latent support, and Wasserstein proximity between the deployment law
and the reference law is what controls task risk.

\section{Problem Formulation}
\label{sec:problem}

\subsection{Notation and Formal Setup}
\label{sec:problem_notation}

We formally define the following notation to accommodate the heterogeneity of the operator-defined domains we consider; throughout, graph-structured systems serve as the concrete instantiation, with the graph operator playing the role of the system operator $\GSO$ in the general setup. Let $\mathcal{G}^{[k]} = (\mathcal{V}^{[k]},
\mathcal{E}^{[k]}, \mathbf{W}^{[k]})$ denote the $k$-th graph, with
$N_k = |\mathcal{V}^{[k]}|$ nodes and weighted adjacency $\mathbf{W}^{[k]}$.
Associated with each graph is a domain-specific graph operator $\GSO^{[k]}$,
which may be real- or complex-valued depending on the application: normalized
Laplacians, diffusion operators, admittance matrices, attention-weighted
message-passing operators, and other graph shift operators are all admissible.

\textbf{Observation spaces.}
Let $\mathcal{X} = \{\mathcal{X}^{(1)}, \mathcal{X}^{(2)}, \mathcal{X}^{(3)}\}$
denote the collection of observation spaces available at deployment. Concretely:
\begin{itemize}
  \item $\mathcal{X}^{(1)} \subset \R^{F_x}$: sparse or incomplete node-level
  measurements, available only at observed nodes $\{i : m_i^{[k]} = 1\}$;
  \item $\mathcal{X}^{(2)} \subset \{0,1\}^{N_k}$: the observation mask
  $m^{[k]}$ indicating which nodes are observed;
  \item $\mathcal{X}^{(3)} \subset \R^{F_p}$: graph structural descriptors
  $p_i^{[k]}$ derived from $\GSO^{[k]}$, such as eigenvectors, diffusion
  coordinates, or positional encodings.
\end{itemize}
Throughout we denote by $\mathbf{x}^{[k]} = (x^{[k]}, m^{[k]}, p^{[k]})$ the observation
deployment tuple for graph $k$ where $x^{[k]}$ denotes its measurement component. The structural descriptor
$p^{[k]}$ is the sole carrier of topology information at deployment; no
other graph-specific quantity is available at test time.

\textbf{Semantic spaces.}
Let $\mathcal{S}$ denote the semantic space of privileged signals available
only during training. The semantic signal
\begin{equation}
s^{[k]} \in \mathbb{K}^{N_k \times F_s}
\;\;\text{or}\;\;
s^{[k]} \in \mathbb{K}^{F_s},
\qquad \mathbb{K} \in \{\R, \Cset\},
\end{equation}
represents the complete information state of the system at training time:
the dense, privileged signal from which any task target $T^{[k]}$ defined
on $\mathcal{G}^{[k]}$ can in principle be derived. Formally, $s^{[k]}$
is a sufficient statistic for the task family in the sense that, given
$s^{[k]}$ and the graph operator $\GSO^{[k]}$, every task target
$T^{[k]}$ is computable, whereas given the observation tuple
$\mathbf{x}^{[k]}$ alone it is not. The signal $s^{[k]}$ is dense and
complete by construction: it encodes the unobserved states, the latent
relational structure, and the privileged information that $\mathbf{x}^{[k]}$
cannot recover. It is available in full during training and discarded
entirely at deployment. 

\begin{remark}[Terminology: privileged vs. semantic]
\label{rem:terminology}
Throughout the paper we use two adjectives for $s^{[k]}$ that emphasize
different properties. \emph{Privileged} refers to the access
asymmetry: $s^{[k]}$ is available during training but discarded at
deployment, the property that motivates the asymmetric architecture
of Section~\ref{sec:method}. \emph{Semantic} refers to the content
sufficiency just stated: $s^{[k]}$ encodes enough information to
determine the task target $T^{[k]}$ on $\mathcal{G}^{[k]}$, the
property that the geometric guarantees of
Section~\ref{sec:geometry} rely on. The two are independent: a signal
can be privileged without being semantic, or semantic without being
privileged. Both are needed for the framework to function. We use
``privileged'' when emphasizing access (e.g., ``privileged pathway,''
``$s$ is discarded at deployment''), ``semantic'' when emphasizing
content structure (e.g., ``semantic manifold $\Ms$,'' ``semantic
reference law $\mu_\star^s$''), and either when both properties are
simultaneously in view.
\end{remark}

\textbf{Latent space and paired data.}
Let $\mathcal{Z} \subset \R^d$ denote the universal latent space containing
the latent manifold $\Mz \subset \mathcal{Z}$, where $z \in \Mz$ represents
a learned graph-agnostic task coordinate. 
For each graph $k$, we treat each target as a map $T^{[k]} : \Ms^{[k]} \to \mathcal{Y}^{[k]}$, where $\mathcal{Y}^{[k]}$
denotes the per-graph target space of the task (e.g., a discrete
label set $\{1,\dots,C\}$ for node classification, or
$\mathbb{R}^{N_k}$ for node-level regression). We treat $\mathcal{Y}^{[k]}$ as a normed space throughout; for classification, $\|\cdot\|$ is understood on the logit/score representation.
Also, for each $k$, the observation
tuple and semantic signal induce latent random variables
\begin{equation}
\zx^{[k]} = E_x(\mathbf{x}^{[k]}), \qquad
\zs^{[k]} = E_s(s^{[k]}),
\end{equation}
with induced laws $\mu_x^{[k]} = \mathrm{Law}(\zx^{[k]})$ and
$\mu_s^{[k]} = \mathrm{Law}(\zs^{[k]})$. 
Note that we use $E_s$ and $E_x$ for the encoders as mathematical maps and
$\Ps := E_s$, $\Px := E_x$ for their parameterized realizations as
network pathways (Section~\ref{sec:method}); the two notations denote
the same underlying objects, with the choice signaling whether we are
reasoning about the map or about its implementation.
These laws are the central
objects in the geometric transfer conditions specified in
Section~\ref{sec:geometry} and the verification protocol of
Section~\ref{sec:verify}. The teacher encoder $E_s$ consumes $s^{[k]}$ to shape the latent geometry; the deployment encoder $E_x$ never observes $s^{[k]}$ directly but is trained to reach the latent region $E_s$ defines.
The training dataset for graph $k$ consists
of paired samples
$\mathcal{D}^{[k]} = \{(\mathbf{x}_i^{[k]}, s_i^{[k]})\}_{i=1}^{n_k}$.

\subsection{The Cross-Instance Transfer Problem}
\label{sec:problem_transfer}

\textbf{Cross-instance transfer.}
Given training graphs $\{\mathcal{G}^{[k]}\}_{k=1}^{K_\text{tr}}$ with
paired data $\mathcal{D}^{[k]}$, we seek a deployment encoder $E_x$ and
a task decoder $D$,
\begin{equation}
\begin{aligned}
E_x &: \mathbf{x}^{[k]} \mapsto z \in \Mz, \\
D   &: (z,\, N) \mapsto \hat{s},
\end{aligned}
\end{equation}
such that for any unseen test graph $\mathcal{G}^{[k']}$ with
$k' > K_\text{tr}$, the estimator
\begin{equation}
\hat{s}^{[k']}
=
D\!\left(E_x\!\left(\mathbf{x}^{[k']}\right),\, N_{k'}\right),
\end{equation}
with $E_x$ and $D$ fixed after training, produces accurate task outputs
on $\mathcal{G}^{[k']}$. This problem is hard for three compounding
reasons: the dimensionality $N_{k'}$ may be new; the operator
$\GSO^{[k']}$ induces a new spectral or message-passing basis; and node
indices and task semantics need not align across topologies.
Consequently, pointwise transfer is ill-defined, and any admissible solution
must operate at the level of the underlying signal \emph{manifold},
not individual graph coordinates.

\textbf{Running example.}
In the power-system case study 
$\mathcal{G}^{[k]}$ is a distribution-network topology, $\GSO^{[k]}$ is the
complex bus-admittance matrix $\mathbf{Y}^{[k]}$, and the privileged signal
$s^{[k]}$ is the full bus state (complex voltages and injections)
generated by a power-flow simulator that plays the role of the teacher.
The deployment tuple $\mathbf{x}^{[k]}$ comprises sparse AMI measurements
$x^{[k]}$ at a subset of buses, the observation mask $m^{[k]}$, and
Laplacian-based positional descriptors $p^{[k]}$ derived from
$\mathbf{Y}^{[k]}$. The task target $T^{[k]}$ is the bus-state estimate at
the unobserved nodes.

\subsection{A Bayesian Reading: The Teacher as a Revealed E-Step}
\label{sec:bayesian}

Before committing to the architecture that realizes this setup
(Section~\ref{sec:method}), it is worth pausing on \emph{why} the
asymmetric pathway should be expected to help, expressed in
inference-theoretic rather than architectural terms. The reading is
illustrative: the latent $z$ and both encoders are learned end-to-end
by supervised training, not derived from a parametric likelihood. What
follows is best understood as a fable that names a failure mode and
its structural remedy, not a probabilistic model the method commits to.

Consider the dependency structure depicted in
Figure~\ref{fig:teacher_intuition}~(left). A regime variable $\theta$
,  which instance or operating point we are on, gives rise to a
latent task code $z \in \Mz$, which in turn produces three observable
views: the deployment observation $\mathbf{x}$, the privileged signal
$s$, and the task target $y$. The graph is symmetric in $\mathbf{x}$,
$s$, and $y$; what is asymmetric is the \emph{observation pattern},
with all three observed at training but only $\mathbf{x}$ at deployment.
An encoder trained on $(\mathbf{x}, y)$ pairs alone can recover, at
best, the marginal $\mathbb{E}[y \mid \mathbf{x}]$, which averages over
latent modes that share the same observation. Modes geometrically
distinct in $z$ but projecting to overlapping regions in $\mathbf{x}$
are blurred together (Figure~\ref{fig:teacher_intuition}, right). The
weights of this averaging are a property of the \emph{training prior}
on $z$, so under a new operator $\GSO^{[k']}$ the same $\mathbf{x}$ has
different modes underneath it and the learned averaging is wrong , 
exactly the failure mode of single-encoder fine-tuning identified in
Section~\ref{sec:intro}.

\begin{figure}[t]
  \centering
  \includegraphics[width=0.85\linewidth]{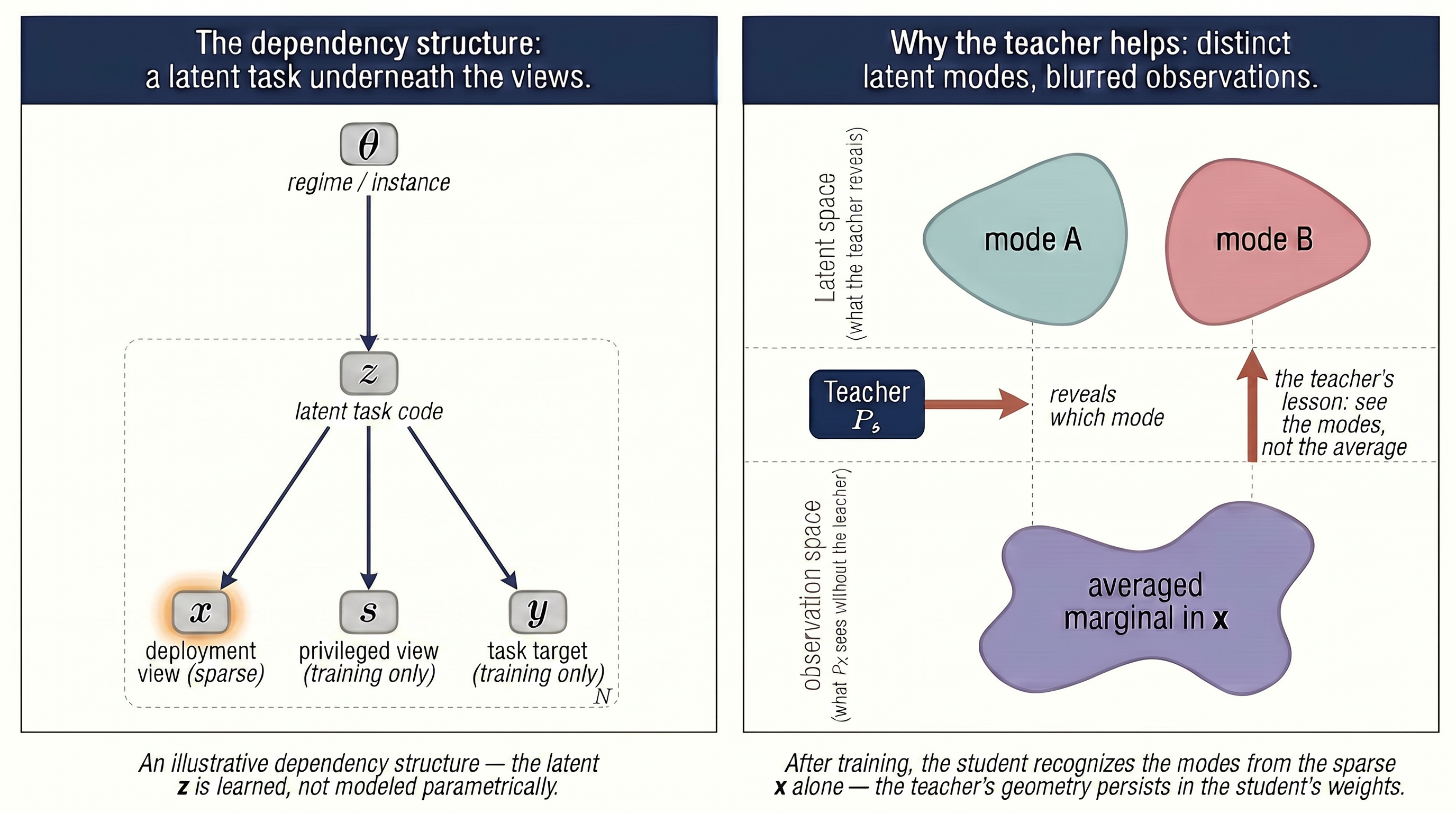}
  \caption{An intuition aid for the teacher--student asymmetry.
  \textbf{Left:} the dependency structure used informally throughout
  this section. A regime $\theta$ generates a latent task code $z$,
  which in turn produces three views: the sparse deployment observation
  $\mathbf{x}$, the dense privileged signal $s$, and the task target
  $y$. Only $\mathbf{x}$ is observed at deployment. \textbf{Right:} the
  practical consequence of the asymmetry. Distinct latent modes in $z$
  can project to overlapping regions in $\mathbf{x}$, so an encoder
  trained on $\mathbf{x}$ alone fits the averaged marginal. The
  privileged signal $s$ lets the teacher resolve the modes; the student
  is trained to reach the same resolved geometry from $\mathbf{x}$
  alone. The graphical model is an illustrative reading: the latent $z$
  and both encoders are learned by end-to-end supervised training, not
  derived from a parametric likelihood.}
  \label{fig:teacher_intuition}
\end{figure}

The privileged pathway resolves this averaging. When the dense signal
$s$ is informationally sufficient to localize $z$ given the regime
$\theta$, the role assigned to $s$ by the sufficiency property
above, where every task target on $\mathcal{G}^{[k]}$ is computable
from $s^{[k]}$ and $\GSO^{[k]}$, the teacher encoder $E_s$
identifies which mode the sample belongs to rather than averaging over
the posterior. In the language of EM, $E_s$ realizes a \emph{revealed
E-step}; the consistency loss $\|\zx - \zs\|$ that pulls the student's
latent code to the teacher's then plays the role of the M-step under
revealed latents. This is the inference-theoretic content of the
lineage discussed in Section~\ref{sec:related}: LUPI, FitNets, and
cross-modal distillation all supply a teacher with information the
student does not see, and in each case the student inherits the
resolved geometry by alignment in latent space rather than by marginal
averaging. What is new in our setting is that the teacher's resolution
is \emph{operator-anchored}: $s$ is not merely informationally richer,
it is richer in a way that is consistent across instances, because the
operator-polynomial features in the teacher inherit their stability
from the spectrum of $\GSO$. Whether the student's latent representation,
learned to match the teacher on training instances, in fact lands in
the same region on a previously unseen operator is the condition
Section~\ref{sec:geometry} formalizes quantitatively, in terms of
Wasserstein proximity between latent laws.

A closing remark. At deployment the regime $\theta$ is generally
unknown, so the genuinely Bayesian object is the \emph{soft} E-step
,  a posterior over $\theta$ given $\mathbf{x}$, with the predictor
recovered as an integral over within-regime experts. This is the
mixture-of-experts deployment view familiar from switching-state-space
models, time-varying channel estimation, and load-profile-conditioned
inference; we sketch it as a future direction in
Section~\ref{sec:conclusion} but do not develop it here. The version we
develop is the simpler one: a single student trained to land near the
teacher's resolved latent on every training instance, with the
condition for transfer to unseen instances stated geometrically in the
next section.

\section{Method: Asymmetric Dual-Path Architecture and Training Objectives}
\label{sec:method}

\subsection{Architecture Overview}
The proposed architecture is asymmetric and deliberately
application-agnostic; see Figure~\ref{fig:architecture} for an overview.
We suppress the graph index $[k]$ in this section where unambiguous; all
encoders and the decoder are shared across graphs. $\Ps$ processes the
dense privileged signal $s$ available only during training, using a
graph encoder whose features are polynomials in the system operator
$\GSO$ and whose coefficients are stable under spectral perturbation
(Section~\ref{sec:geometry}). $\Px$ processes only the sparse
observations $\{x_i : m_i = 1\}$ together with graph structural
descriptors derived from $\GSO$ such as eigenvectors, positional
encodings, diffusion coordinates, or other topology-aware node features.
Both paths project to the same latent manifold $\Mz \subset \R^d$.
A shared adaptive decoder reconstructs structured outputs on graphs of
arbitrary size. For node-level regression or reconstruction tasks, the
decoder outputs per-node states; for graph-level prediction tasks, the
same latent can feed a pooled prediction head. $\Ps$ is used
\emph{only during training} to shape the latent geometry; at deployment
only $\Px$ and the shared decoder $D$ are retained.

\begin{figure*}[t]
  \centering
  \includegraphics[width=0.7\linewidth]{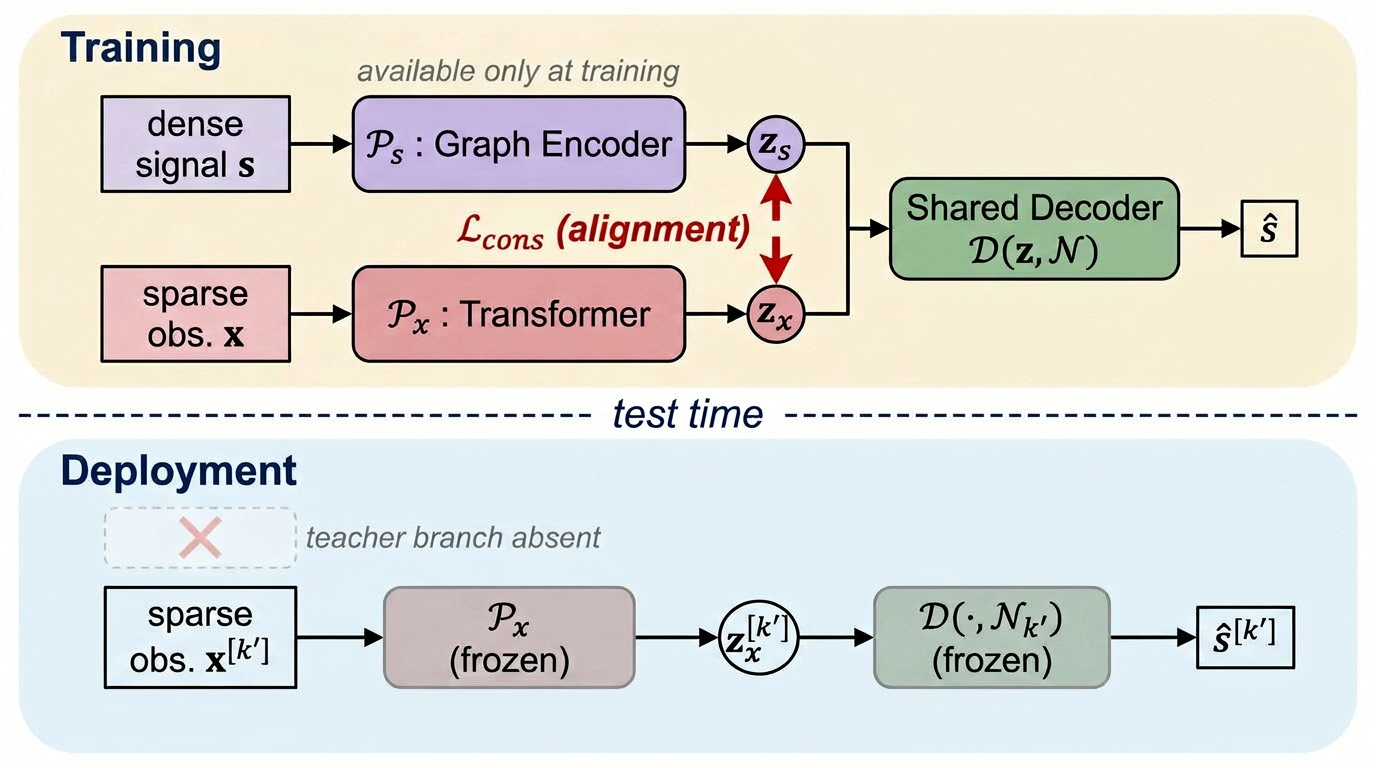}
  \caption{Asymmetric dual-path architecture.
  \textbf{Left (training phase):} path~$\Ps$ (violet) encodes the dense privileged signal $s^{[k]}$
  into latent code $z_s^{[k]}$; path~$\Px$ (rose) encodes the sparse observation tuple
  $\mathbf{x}^{[k]}=(x^{[k]},m^{[k]},p^{[k]})$ into latent code $z_x^{[k]}$.
  A geometric consistency loss $\mathcal{L}_\text{cons}$ pulls the two codes
  together, causing $\Ps$ to act as a teacher that shapes the latent geometry
  that $\Px$ cannot see directly.  Both codes are decoded by the shared adaptive
  decoder $D(z,N)$, supervised by reconstruction losses $\mathcal{L}_\text{rec}^s$
  and $\mathcal{L}_\text{rec}^x$.
  \textbf{Right (deployment):} $\Ps$ and the privileged signal $s$ are discarded.
  The frozen $\Px$ maps the sparse observations of an unseen graph~$k'$ to~$z_x^{[k']}$,
  which $D(\cdot,N_{k'})$ decodes at arbitrary size $N_{k'}$.
  Inference cost is ${\sim}0.8$~ms independent of graph size.}
  \label{fig:architecture}
\end{figure*}

\subsection{Encoders and Shared Adaptive Decoder}

\noindent\textbf{Path S: Teacher Encoder (Operator-Polynomial Features).}
$\Ps$ is a graph encoder operating on the dense privileged signal $s$.
Its features are polynomials in the system operator $\GSO$, a choice
that is load-bearing for the transfer guarantees of
Section~\ref{sec:geometry}: the coefficients of the polynomial
representation are stable under spectral perturbation of $\GSO$, so the
latent geometry shaped by $\Ps$ remains comparable across instances
with different but related operators. A general spectral-localized
form is
\begin{equation}
H^{(\ell+1)} = \sigma_s\!\Bigl( \sum_{k=0}^{K-1} \mathcal{T}_k(\GSO)\,H^{(\ell)}\,\Theta_k^{(\ell)} \Bigr),
\label{eq:cheb}
\end{equation}
with $H^{(0)} = s$ and learnable filter weights $\Theta_k^{(\ell)}$. The
operator family $\{\mathcal{T}_k(\GSO)\}$ can denote Chebyshev
polynomials, diffusion powers, localized message-passing kernels, or
other spectral graph filters. This form covers both real-valued and
complex-valued graph signals. In domains with complex physics, such as
AC power flow or wave propagation, the weights and multiplications
remain complex-valued so that amplitude and phase interact throughout
the encoder. In our power-grid instantiation we use Chebyshev
polynomials of a complex Laplacian with split complex ReLU,
\begin{equation}
\sigma_{\Cset}(z) = \mathrm{ReLU}(\mathrm{Re}\,z) + j\,\mathrm{ReLU}(\mathrm{Im}\,z).
\end{equation}
After the graph-encoding stack, node features are pooled through an
adaptive graph pooling layer and mapped through an MLP to $\R^d$. The
output $\zs \in \R^d$ is the teacher's view of the latent manifold.

\noindent\textbf{Path X: Structure-Aware Transformer (Deployment).}
$\Px$ operates only on observed nodes. For each graph we derive a structural descriptor $p_i \in \R^{F_p}$ from the graph operator (for example Laplacian or diffusion eigenvectors, random-walk coordinates, shortest-path positional encodings, or domain-specific spectral embeddings). These descriptors are concatenated with the observed node features to yield
\begin{equation}
\tilde{x}_i = \big[\,x_i\;\big|\; p_i\,\big] \in \R^{F_x + F_p},
\end{equation}
for each $i$ with $m_i = 1$. The variable-length set of observed-node features is embedded via a linear projection to hidden dimension $h$, summed with learnable positional embeddings, and processed by a transformer or any permutation-equivariant sparse-set encoder. Because different graphs contribute different numbers of observed nodes, the transformer's per-token outputs are averaged only over the observed (non-padded) positions, a {\it masked global mean pool} that yields a single graph-level vector, which 
a two-layer Multi-Layer Perceptron (MLP) head then projects to the latent $\zx \in \R^d$.

The role of the structural descriptor deserves further clarification. $\Px$ does not have access to the dense privileged signal, so topology information must be injected explicitly through graph-aware node descriptors. In our power-system instantiation, $p_i$ is formed from $\Ybus$ eigenvectors; in other domains it may be derived from Laplacian, diffusion, or other topology-aware coordinates.

\noindent\textbf{Shared Adaptive Decoder.}
Both latents share a single decoder $D : (z, N) \mapsto \hat{s}$. For structured node-level prediction, the decoder projects $z \in \R^d$ to a memory token in $\R^h$, instantiates $N$ learnable positional queries $Q \in \R^{N \times h}$ drawn from a shared pool supporting up to $N_\text{max}$ nodes, and applies a transformer decoder to generate per-node outputs. The same weights are applied regardless of $N$, and the same decoder is applied to both $\zx$ and $\zs$ during training. For graph-level tasks, the same latent manifold can instead feed a pooled classification or regression head, while retaining the same training and verification logic.

\subsection{Training Objectives}
We train the full framework with a composite objective
\begin{equation}
\Loss = \Loss_\text{rec}^x + \Loss_\text{rec}^s
      + \lambda_c \Loss_\text{cons} + \lambda_r \Loss_\text{reg},
\label{eq:loss}
\end{equation}
where each term targets a specific requirement of Theorems~\ref{thm:compat}--\ref{thm:transferbound}.

\noindent\textbf{Reconstruction (teacher-path fidelity).}
We enforce both latents to decode to the training-time target signal, using a task-appropriate loss,
\begin{equation}
\Loss_\text{rec}^x = \ell(\hat{s}_x, s),
\qquad
\Loss_\text{rec}^s = \ell(\hat{s}_s, s),
\end{equation}
where $\hat{s}_x = D(\zx, N)$, $\hat{s}_s = D(\zs, N)$, and $\ell(\cdot,\cdot)$ may be squared error, complex-valued reconstruction loss, cross-entropy, or any task-specific surrogate. In our power-grid instantiation, $\ell$ is squared error on stacked real and imaginary phasor components. This term directly controls the second term in the transfer bound \eqref{eq:transfer-bound}: if the teacher pathway cannot realize the task accurately, there is nothing stable to transfer.

\noindent\textbf{Cross-modal consistency (geometric co-location).}
$\Px$ must reach the same latent region as $\Ps$ so that, at deployment, the decoder sees a latent vector consistent with the teacher-shaped manifold. We use a cosine consistency with a small Euclidean term,
\begin{equation}
\Loss_\text{cons} = \left(1 - \frac{\zx^\top \zs}{\|\zx\|_2\,\|\zs\|_2}\right)
                 + \gamma \, \|\zx - \zs\|_2^2,
\label{eq:cons}
\end{equation}
with $\gamma$ small. Cosine similarity permits magnitude differences arising from the asymmetric information content of the two paths while enforcing directional alignment; the Euclidean term tightens co-location in magnitude-sensitive regimes. This term directly targets condition (c) of Section~\ref{sec:geometry}. In the inference-theoretic reading of Section~\ref{sec:bayesian}, $\Loss_\text{cons}$ plays the role of an M-step under revealed latents: the teacher's encoder $\Ps$ supplies the resolved latent code that $\Px$ would otherwise have to recover by marginalizing over modes it cannot distinguish from $\mathbf{x}$ alone.

\begin{remark}
    We note that the transfer bound of Theorem~\ref{thm:transferbound}
is stated in the ambient $\R^d$ Wasserstein geometry, while the
cosine term in \eqref{eq:cons} is angular. The two metrics agree
up to a magnitude-correction term whenever $\|\zx\|$ and $\|\zs\|$
lie in a bounded range, which is precisely the regime that the
regularizer $\Loss_\text{reg}$ enforces. The cosine consistency
therefore serves as a numerically stable surrogate for Euclidean
alignment within the magnitude-bounded latent region; the small
Euclidean term with weight $\gamma$ is retained to control the
residual magnitude correction explicitly.
\end{remark}

\noindent\textbf{Regularization and noise injection.}
$\Loss_\text{reg} = \|\zx\|_2^2 + \|\zs\|_2^2$ bounds the ambient radius of the learned manifold, improving optimization stability. To promote continuity and discourage memorization of observation-specific patterns, during training we add Gaussian noise $\eta \sim \mathcal{N}(0, \sigma^2 I)$ to observed node features $x_i \leftarrow x_i + \eta$ for $i$ with $m_i = 1$ only.

\noindent\textbf{Relation to local consistency and normalization.}
An extended objective form includes an explicit local-consistency term
\begin{equation}
\Loss_\text{local} = \sum_{i} \sum_{j \in \mathcal{N}_\varepsilon(i)}
                    \|E(x_i) - E(x_j)\|_2^2,
\label{eq:llocal}
\end{equation}
which penalizes the encoder for tearing $\varepsilon$-neighborhoods in the data manifold. For operator-driven domains with strong graph locality, such as the power-grid instantiation studied here, this term can be omitted because localized spectral filters already enforce smoothness and the optional local-geometry diagnostics in Section~\ref{sec:geometry} can be checked separately if needed. The training objective \eqref{eq:loss} is therefore a simplified instance of the broader framework. All observation and privileged-signal modalities are normalized using training-set statistics appropriate to the task. In the power-system instantiation, AMI features are standardized per channel using statistics computed over \emph{sensor buses only}, while voltage real and imaginary components are standardized separately over all training buses. All losses are computed in normalized space; evaluation metrics are reported after denormalization when needed.

\noindent\textbf{Deployment.}
At test time $\Ps$ and the privileged signal $s$ are discarded. Given only $(x, m, \text{structural descriptors})$ for an unseen graph of size $N_\text{test}$,
\begin{equation}
\hat{s} = D(E_x(x, m, \text{structural descriptors}),\; N_\text{test}).
\end{equation}
For graph-level tasks, the same deployment latent can be fed directly to a prediction head instead of a node decoder. Inference cost is dominated by the sparse-observation encoder and the task head, and remains independent of graph size at fixed observation count. Each deployment-time forward pass is accompanied by the per-instance transfer certificate developed in Section~\ref{sec:verify}, which gates whether the result falls inside the architecture's transfer envelope.

\section{Geometric Foundation and Verification Protocol}
\label{sec:geometry}

Section~\ref{sec:method} described the asymmetric architecture; this
section establishes when and why it transfers. We develop three layers
of analysis. \S\ref{sec:geom_pointwise} gives a pointwise geometric
criterion (\emph{transfer-compatibility}) that asks when transfer
is well-posed in principle. \S\ref{sec:geom_law} converts geometric
proximity into bounded task risk via Wasserstein stability between
latent laws, pinpointing what must hold for a new instance to fall
inside the transfer envelope. \S\ref{sec:verify} translates the
latent-law conditions into a finite-sample certification protocol that
issues a quantitative pass/fail decision for each unseen instance
from deployment-side data alone. Read against the Bayesian fable of
\S\ref{sec:bayesian}, this section says quantitatively when the
teacher's revealed latent geometry survives the move to a new
operator.

\subsection{Transfer Compatibility and Shared Latent Coordinates}
\label{sec:geom_pointwise}

We begin with a structural hypothesis that the privileged signals on
different instances are not arbitrary but lie on related geometric
objects.
\begin{assumption}[Graph manifold hypothesis]
\label{assump:manifold}
For each graph $\mathcal{G}^{[k]}$, the semantic signals generated by
the underlying physical, relational, or generative process concentrate
near a low-dimensional smooth submanifold $\Ms^{[k]}$ of their ambient
space, with deviations from the manifold treated as observation noise
\citep{bengio2013representation,fefferman2016testing}. Across graphs,
the submanifolds $\{\Ms^{[k]}\}$ may share a common intrinsic dimension
and topological type, even though the ambient dimensions $N_k$ differ.
\end{assumption}

Under Assumption~\ref{assump:manifold}, although the ambient
dimension $N_k$ varies with topology, the intrinsic manifolds
$\{\Ms^{[k]}\}$ may share a common generating mechanism. 
%
A latent space suitable for transfer must embed all $\{\Ms^{[k]}\}$ into a single connected geometric object $\Mz$ while preserving their local structure. The central issue is not whether two graphs have the same adjacency matrix, but whether their tasks can be written in a common latent coordinate system. We formalize this idea through \emph{transfer compatibility}: graph learning can transfer when topology variation is absorbed as a continuous change of coordinates in latent space rather than as a change of task.


\begin{definition}[Transfer-Compatible Graph Family]
\label{def:compat}
Let graph $k$ have observation manifold $\Mx^{[k]}$, semantic
manifold $\Ms^{[k]}$, and target map
$T^{[k]} : \Ms^{[k]} \to \mathcal{Y}^{[k]}$. A family
$\{\mathcal{G}^{[k]}\}_{k=1}^{K}$ is \emph{transfer-compatible}
with respect to a latent manifold $\Mz$ if there exist
\begin{enumerate}
  \item a connected subset $U \subseteq \Mz$ shared across graphs,
  \item semantic encoders $E_s^{[k]} : \Ms^{[k]} \to U$ that are
        homeomorphisms onto $U$, and
  \item a shared decoder family $D(\cdot, N)$ such that for each
        graph size $N_k$,
  \begin{equation}
  T^{[k]}(s) = D(E_s^{[k]}(s), N_k), \qquad \forall s \in \Ms^{[k]}.
  \end{equation}
\end{enumerate}
\end{definition}

\begin{theorem}[Shared Latent Coordinates Imply Graph Transfer]
\label{thm:compat}
If $\{\mathcal{G}^{[k]}\}_{k=1}^{K}$ is transfer-compatible in the sense of Definition~\ref{def:compat}, then for any pair of graphs $k$ and $k'$, the semantic manifolds are related on the shared latent set $U$ by the coordinate transform
\begin{equation}
\Psi^{[k \to k']} = \big(E_s^{[k']}\big)^{-1} \circ E_s^{[k]},
\end{equation}
which is a homeomorphism between $\Ms^{[k]}$ and $\Ms^{[k']}$. Consequently, all graphs in the family share the same latent task coordinates, and a single decoder family $D(\cdot, N)$ is sufficient across topologies. In other words, graph transfer is well-posed whenever topology variation changes coordinates but does not change the latent task itself.
\end{theorem}
\noindent\emph{Proof sketch: see Appendix~\ref{app:geom_proofs} for the full proof.}


\begin{theorem}[Why the Deployment Encoder Transfers]
\label{thm:transferbound}
Assume the setting of Theorem~\ref{thm:compat}, and let
$E_x^{[k]} : \Mx^{[k]} \to U$ be the deployment encoder for graph
$k$. Assume Definition~\ref{def:compat} holds approximately, in the
sense that there exists $\varepsilon_T \ge 0$ such that
\[
  \big\| D(E_s^{[k]}(s), N_k) - T^{[k]}(s) \big\| \le \varepsilon_T
  \qquad \forall s \in \Ms^{[k]}.
\]
If $D(\cdot, N)$ is uniformly $L_D$-Lipschitz on $U$
for all graph sizes $N$ under consideration (including
any unseen graph $\mathcal{G}^{[k']}$ to be evaluated),
then for any paired sample $(x^{[k]}, s^{[k]})$,
\begin{equation}
\begin{aligned}
\big\| D(E_x^{[k]}(x^{[k]}), N_k)&
      - T^{[k]}(s^{[k]}) \big\|\le\\
&L_D \big\| E_x^{[k]}(x^{[k]}) - E_s^{[k]}(s^{[k]}) \big\| +\varepsilon_T.
\end{aligned}
\label{eq:transfer-bound}
\end{equation}
Thus transfer succeeds when the teacher-path error and the
latent alignment error are both small. For unseen graphs, the
same bound applies provided: the new graph is
$\varepsilon_T$-approximately transfer-compatible with latent
subset $U$, its deployment encoder $E_x^{[k']}$ maps
into $U$, and $D(\cdot, N_{k'})$ is $L_D$-Lipschitz on $U$.
\end{theorem}
\noindent\emph{Proof sketch: see Appendix~\ref{app:geom_proofs} for the full proof.}

Theorem~\ref{thm:compat} answers \emph{when} graph transfer is possible: transfer is justified only when different graph tasks admit a common latent coordinate system. Theorem~\ref{thm:transferbound} answers \emph{why} the learned deployment path transfers: once the sparse observation encoder reaches the same latent region as the teacher, decoder continuity converts latent alignment into bounded task error. Together with the latent-law results below, these statements reduce graph transfer learning to three measurable requirements: (a) \emph{reconstruction fidelity on the teacher and deployment pathways}, (b) \emph{latent-law proximity to a shared semantic reference}, and (c) \emph{cross-modal alignment between teacher and deployment latents}.

\begin{remark}[Justification in physically structured domains]
    Assumption~\ref{assump:manifold} is well motivated whenever the
governing operator $\GSO^{[k]}$ varies smoothly within the graph
family. Weighted Hermitian operators such as the Laplacian or the bus
admittance matrix $\Ybus^{[k]}$ live in a Euclidean space, and edge
addition or removal can be realized as a continuous edge-weight sweep.
By Weyl's inequality, the spectrum of such operators is Lipschitz
continuous in the operator norm \citep{bhatia1997matrix}, and any
polynomial or Chebyshev graph filter is smooth in the operator, so
filtered semantic signals trace out a smoothly parameterised family
in the ambient signal space. Spectral positional descriptors derived
from $\GSO^{[k]}$ are themselves available both at training and at
deployment, so any sharp transitions in the eigenvector basis at
spectral-gap closures are observed inputs to $\Px$ rather than hidden
coordinate ambiguities; the encoder is trained on the resulting
diversity of bases and learns approximate invariance to the
sign-and-basis ambiguity intrinsic to spectral descriptors. (For
graph families with persistent spectral degeneracies, an explicitly
sign- and basis-equivariant architecture
\citep{lim2023signbasis} would replace this learned invariance with a
structural one; this is not required in our case study but is a
natural extension.) The deformations between graphs in our
power-system case study sit in this regime, which makes the graph
manifold hypothesis a natural rather than an idealised assumption for
the experiments of \S\ref{sec:exp}.
\end{remark}

\begin{remark}[From Ideal Conditions to Empirical Diagnostics]
Theorems~\ref{thm:compat}--\ref{thm:transferbound} are structural statements: they explain why transfer is possible when a graph family shares latent task coordinates and when the deployment encoder aligns with the teacher pathway. In practice, neither transfer compatibility nor the bound in \eqref{eq:transfer-bound} can be verified exactly from finite samples. Section~\ref{sec:verify} therefore introduces diagnostics that look for the empirical signatures of these conditions: reconstruction fidelity, latent-law Wasserstein proximity, and cross-modal alignment error. A passing diagnostic is evidence consistent with transfer compatibility, but not a formal proof.
\end{remark}

\subsection{Latent-Law Transfer Guarantees}
\label{sec:geom_law}

The previous results are pointwise and geometric: they say when a
deterministic latent code on a single sample produces bounded task
error. We now state the distributional version used by the
verification protocol. The shift in viewpoint is the one anticipated
by the Bayesian fable of \S\ref{sec:bayesian}: the teacher does not
just resolve individual modes, it shapes a \emph{law} on the latent
space, and the student is trained to track that law. Whether the
student succeeds on a new operator therefore reduces to whether its
induced latent law remains close to the teacher's reference law in a
quantifiable sense. The next theorems make ``close'' precise as
Wasserstein proximity.
For graph
$k$, let $x$ and $s$ be the observation and semantic random
variables drawn from the data distribution on graph $k$, and let
\begin{equation}
z_x^{[k]}=E_x^{[k]}(x), \qquad
z_s^{[k]}=E_s^{[k]}(s),
\end{equation}
denote the induced latent random variables. We write
$\mu_x^{[k]}=\mathrm{Law}(z_x^{[k]})$ and
$\mu_s^{[k]}=\mathrm{Law}(z_s^{[k]})$ for their pushforward
distributions on $\Mz$. We introduce two reference latent laws
pooled across the training family $\{\mathcal{G}^{[k]}\}_{k=1}^{K}$:
the \emph{semantic reference law} $\mu_\star^s$, formed from path-S
latent codes, and the \emph{deployment reference law} $\mu_\star^x$,
formed from path-X latent codes:
\begin{equation}
\begin{aligned}
\mu_\star^s
:= \mathrm{Law}\!\left(z_s \text{ pooled over training}\right),\\
\mu_\star^x
:= \mathrm{Law}\!\left(z_x \text{ pooled over training}\right).
\end{aligned}
\end{equation}
Both are estimable from training data. $\mu_\star^s$ defines the
reference task family shaped by the teacher; $\mu_\star^x$ is its
deployable counterpart and is the sole reference needed at test time.
For a decoder $D$ and a fixed graph $k$, define the
\emph{per-graph target map}
\begin{equation}
y_k : U \to \mathcal{Y}^{[k]},
\qquad
y_k(z) \;:=\; T^{[k]}\!\big((E_s^{[k]})^{-1}(z)\big),
\end{equation}
which is well-defined on the shared latent set $U$ because
$E_s^{[k]}$ is a homeomorphism onto $U$ (Definition~\ref{def:compat}).
For graph $k$, $N_k$ is a fixed constant, so the task risk
under a latent law $\mu$ \emph{for graph $k$} is
\begin{equation}
R_\mu^{[k]}(D) \;:=\;
\mathbb{E}_{z\sim\mu}\!\big[\ell\!\left(D(z,N_k),\, y_k(z)\right)\big].
\label{eq:per_graph_risk}
\end{equation}
We abbreviate $R_k(D):=R_{\mu_x^{[k]}}^{[k]}(D)$, the deployment
risk on graph $k$. For the semantic reference law $\mu_\star^s$, the
reference risk is defined analogously as
$R_{\mu_\star^s}(D) := \frac{1}{K}\sum_{k=1}^{K} R_{\mu_\star^s}^{[k]}(D)$,
using each graph's $N_k$ and $y_k$; when the family is
transfer-compatible the shared decoder makes this consistent
across graphs. Throughout this section, $F_D^{[k]}(z) :=
\ell(D(z,N_k), y_k(z))$ denotes the per-graph loss-composition,
and all Lipschitz and Wasserstein statements for graph $k$ are
understood with respect to $F_D^{[k]}$.

\begin{theorem}[Sufficient latent-law transfer condition]
\label{thm:law_sufficient}
For an unseen test graph $k'$, assume:
\begin{enumerate}
  \item \textbf{Semantic decoder fidelity} (training-time verifiable):
  \begin{equation}
  R_{\mu_\star^s}^{[k']}(D)\le \varepsilon_\star^s.
  \end{equation}
  \item \textbf{Graph/path matching calibration}:
  for each graph in the reference family, define
  $c_k:=W_2(\mu_x^{[k]},\mu_s^{[k]})$.  With the same mixture weights
  $\pi_k$ used to form $\mu_\star^x$ and $\mu_\star^s$,
  \begin{equation}
  \left(\sum_k \pi_k c_k^2\right)^{1/2}\le \varepsilon_c .
  \end{equation}
  A sufficient (and more conservative) alternative is $\sup_k c_k\le\varepsilon_c$;
  the sup bound implies the weighted-RMS bound but the converse does not hold.
  This per-graph condition implies
  $W_2(\mu_\star^x,\mu_\star^s)\le\varepsilon_c$ for the pooled
  reference laws.
  \item \textbf{Test-graph deployment proximity} (test-time verifiable):
  \begin{equation}
  W_2\!\left(\mu_x^{[k']},\,\mu_\star^x\right)\le \varepsilon_x.
  \end{equation}
  \item \textbf{Lipschitz loss-composition}: $F_D^{[k']}(z):=\ell(D(z,N_{k'}),y_{k'}(z))$
  is $K$-Lipschitz on the latent region supporting
  $\mu_x^{[k']}$, $\mu_\star^x$, and $\mu_\star^s$.
\end{enumerate}
Then
\begin{equation}
R_{k'}(D) \le \varepsilon_\star^s + K(\varepsilon_c + \varepsilon_x).
\label{eq:law_suff_bound}
\end{equation}
\begin{remark}
Each term is fully verifiable without access to path S on the test
graph. In this reference-law instantiation, $\varepsilon_\star^s$ and
$\varepsilon_c$ are computed on the reference family and need not be
re-evaluated at test time.
$\varepsilon_x$ requires only $\mu_x^{[k']}$ and $\mu_\star^x$, both
formed from deployment-path outputs. The term $\varepsilon_c$
captures graph/path matching calibration per graph: for the same graph
and demand law, the sparse deployment pathway should land near the
full-information teacher pathway. This calibration notion is not a
cross-graph transfer requirement. The actual test-graph transfer
requirement is $\varepsilon_x$, which measures whether the unseen
deployment law remains near the deployment reference.
\end{remark}
\end{theorem}

\begin{figure}[t]
  \centering
  \includegraphics[width=0.6\linewidth]{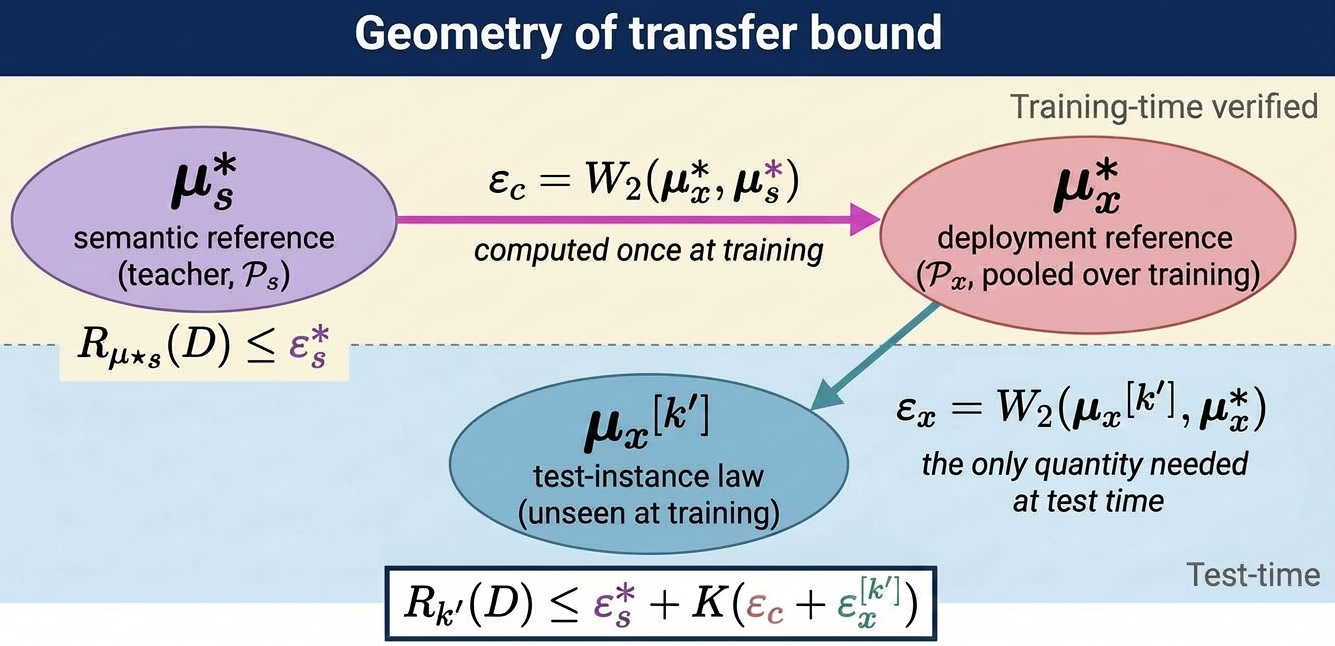}
  \caption{Latent-space geometry underlying Theorem~\ref{thm:law_sufficient},
  with a clear separation between reference-family calibration (warm
  background, top) and what is verified at test time (cool background,
  bottom).
  \textbf{Reference family (top):}
  the semantic reference law $\mu_\star^s$ (purple) is anchored by
  path~$\Ps$ with semantic fidelity $\varepsilon_\star^s$, and the
  graph/path matching calibration constant
  $\varepsilon_c = W_2(\mu_\star^x,\mu_\star^s)$ is obtained from
  paired per-graph path distances $W_2(\mu_x^{[k]},\mu_s^{[k]})$ on
  the reference family.
  \textbf{Test-time (bottom):}
  only the per-instance deployment proximity
  $\varepsilon_x^{[k']} = W_2(\mu_x^{[k']},\mu_\star^x)$ needs to be
  evaluated on the unseen instance $k'$, requiring only path-$\Px$
  outputs. The two hops chain via Lemma~\ref{lem:risk_stability} to
  give the bound
  $R_{k'}(D)\le\varepsilon_\star^s+K(\varepsilon_c+\varepsilon_x^{[k']})$,
  with $\varepsilon_x^{[k']}$ in the theorem statement bounded
  uniformly by $\varepsilon_x$.}
  \label{fig:theorem_geometry}
\end{figure}

\begin{theorem}[Necessary law proximity under identifiability]
\label{thm:law_necessary}
Assume there exists a continuous, non-decreasing function $\psi$ with $\psi(0)=0$ such that
\begin{equation}
R_\mu^{[k]}(D)-R_{\mu_\star^s}^{[k]}(D)
\ge
\psi\!\left(W_2(\mu,\mu_\star^s)\right)
\end{equation}
for every latent law $\mu$ in the relevant family $\{\mu_x^{[k]}, \mu_s^{[k]}, \mu_\star^s\}$, where
$\psi^{-1}(\delta) := \sup\{u \ge 0 : \psi(u) \le \delta\}$. If graph $k$ transfers with relative excess risk
\begin{equation}
R_k(D)-R_{\mu_\star^s}^{[k]}(D)\le \delta,
\end{equation}
then
\begin{equation}
W_2(\mu_x^{[k]},\mu_\star^s)\le \psi^{-1}(\delta).
\end{equation}
If additionally the weighted graph/path calibration condition induces
$W_2(\mu_\star^x,\mu_\star^s)\le \varepsilon_c$ and
$W_2(\mu_x^{[k]},\mu_\star^x)\le \varepsilon_x$, then
\begin{equation}
W_2(\mu_x^{[k]},\mu_\star^s)
\le
\varepsilon_x + \varepsilon_c,
\end{equation}
so the sufficient and necessary conditions share the same
Wasserstein ingredients.
\end{theorem}

\begin{remark}[What is necessary and what is diagnostic]
Theorem~\ref{thm:law_sufficient} gives a sufficient condition: the shared decoder must be accurate on the teacher reference law $\mu_\star^s$, the two pathways must satisfy graph/path matching calibration ($\varepsilon_c$ small), the test-graph deployment law must be close to $\mu_\star^x$ ($\varepsilon_x$ small), and the scalar task loss $F_D^{[k]}=\ell(D(\cdot,N_k),y_k(\cdot))$ must be Lipschitz in latent coordinates. Wasserstein risk stability is then a consequence, not a primitive assumption. The reverse implication in Theorem~\ref{thm:law_necessary} requires identifiability; without it, a decoder can collapse large latent-law differences into nearly identical risks, so successful transfer alone cannot force $W_2$ proximity. Detailed proofs and the resulting near-iff statement are given in Appendix~\ref{app:law_transfer}.
\end{remark}

\begin{corollary}[Probabilistic transfer guarantee]
\label{cor:prob_transfer}
Let $\Pi$ be a distribution over graphs and let $\alpha\in(0,1)$.
Define the $\alpha$-quantile of the per-graph deployment proximity under $\Pi$:
\begin{equation}
\tau_\alpha
:=
\inf\!\left\{
\tau\ge0
:
\Pr_{k\sim\Pi}\!\left[W_2\!\left(\mu_x^{[k]},\mu_\star^x\right)\le\tau\right]
\ge\alpha
\right\}.
\end{equation}
Under the assumptions of Theorem~\ref{thm:law_sufficient},
\begin{equation}
\Pr_{k\sim\Pi}\!\left[R_k(D)\le
  \varepsilon_\star^s+K(\varepsilon_c+\tau_\alpha)\right]\ge\alpha.
\end{equation}
A fraction $\alpha$ of graphs drawn from $\Pi$ achieve transfer error
at most $\varepsilon_\star^s+K(\varepsilon_c+\tau_\alpha)$.
$\varepsilon_\star^s$ and $\varepsilon_c$ are reference-family
constants; $\tau_\alpha$ is estimated from held-out graphs as the
empirical $\alpha$-quantile of
$\bigl\{W_2(\hat\mu_x^{[k]},\hat\mu_\star^x)\bigr\}_{k\in\mathcal{D}_\mathrm{test}}$.
If a deterministic uniform guarantee over a known graph class is desired,
one may instead take
$\varepsilon_x=\sup_k W_2(\mu_x^{[k]},\mu_\star^x)$.  The verification
protocol uses a quantile rather than the empirical maximum, because a
finite-sample maximum is not a reliable upper bound for future unseen
graphs.
\end{corollary}

\begin{proof}
By Theorem~\ref{thm:law_sufficient}, for any graph $k$ with
$W_2(\mu_x^{[k]},\mu_\star^x)\le\tau_\alpha$, we have
$R_k(D)\le\varepsilon_\star^s+K(\varepsilon_c+\tau_\alpha)$.
The event $\{W_2(\mu_x^{[k]},\mu_\star^x)\le\tau_\alpha\}$ has
probability at least $\alpha$ under $\Pi$ by definition of $\tau_\alpha$.
\end{proof}

\begin{remark}[Empirical success rate as a $K$-free verification]
\label{rem:empirical_rate}
The Lipschitz constant $K$ is typically hard to estimate and can
make the bound $\varepsilon_\star+K\tau_\alpha$ loose or vacuous.
A fully empirical alternative bypasses $K$ by reporting the
\emph{transfer success rate} at tolerance $\delta$ directly:
\begin{equation}
\widehat{p}_\delta
:=
\frac{1}{|\mathcal{D}_\mathrm{test}|}
\sum_{k\in\mathcal{D}_\mathrm{test}}
\mathbf{1}\!\left[\hat\varepsilon_x^{[k]}\le\delta\right].
\end{equation}
This is the fraction of test graphs on which deployment error
falls within $\delta$, verifiable directly from simulation without
knowledge of $K$ or $\varepsilon_\star$.
The theoretical bound explains \emph{why} $\widehat{p}_\delta$
should be large when the latent-law conditions hold;
$\widehat{p}_\delta$ itself certifies \emph{that} it is large on
the observed test distribution. Crucially, computing $\widehat{p}_\delta$
and $\hat\tau_\alpha$ requires only $\hat\mu_x^{[k]}$ and
$\hat\mu_\star^x$, both from path X; no path-S access is needed at test time.
Algorithm~\ref{alg:verify} is therefore reformulated below to
report $\widehat{p}_\delta$ and the empirical quantile
$\hat\tau_\alpha$ rather than a binary PASS/FAIL verdict.
\end{remark}
\subsection{Bound Verification Protocol}
\label{sec:verify}

Direct verification of the infinite-sample conditions in Theorem~\ref{thm:law_sufficient} is intractable. The finite-sample certificate therefore computes only the quantities needed by the bound: the graph/path calibration constant $\hat W_c$, the deployment-law proximity $\hat W_x^{[k]}$, the empirical deployment error $\hat\varepsilon_x^{[k]}$, and the resulting bound/pass rate. When path S is available in a controlled study, teacher-access quantities such as $\hat W_s^{[k]}$ may be reported as auxiliary diagnostics, but they are not required by the deployment certificate.

For graph $k$, let $\{(x_i^{[k]},s_i^{[k]},y_i^{[k]})\}_{i=1}^{n_k}$ be held-out paired samples, where $y_i^{[k]}$ denotes the training target associated with $s_i^{[k]}$. Define
\begin{equation}
z_{x,i}^{[k]}=E_x(x_i^{[k]}),
\qquad
z_{s,i}^{[k]}=E_s(s_i^{[k]}),
\end{equation}
the empirical latent laws
\begin{equation}
\hat\mu_x^{[k]}
=
\frac{1}{n_k}\sum_{i=1}^{n_k}\delta_{z_{x,i}^{[k]}},
\qquad
\hat\mu_s^{[k]}
=
\frac{1}{n_k}\sum_{i=1}^{n_k}\delta_{z_{s,i}^{[k]}},
\end{equation}
and two reference-law estimators
\begin{equation}
\hat\mu_\star^s
=
\frac{1}{\sum_k n_k}
\sum_{k=1}^{K}\sum_{i=1}^{n_k}\delta_{z_{s,i}^{[k]}},
\qquad
\hat\mu_\star^x
=
\frac{1}{\sum_k n_k}
\sum_{k=1}^{K}\sum_{i=1}^{n_k}\delta_{z_{x,i}^{[k]}}.
\end{equation}

\subsubsection*{Path Fidelity}
\begin{itemize}
  \item \textbf{Semantic reconstruction fidelity.} Define
  \begin{equation}
  \hat\varepsilon_s^{[k]}
  =
  \frac{1}{n_k}
  \sum_{i=1}^{n_k}
  \ell\!\big(D(z_{s,i}^{[k]},N_k),y_i^{[k]}\big).
  \end{equation}
  This checks whether $\Ps$ realizes the task on held-out samples.
  \item \textbf{Deployment reconstruction fidelity.} Define
  \begin{equation}
  \hat\varepsilon_x^{[k]}
  =
  \frac{1}{n_k}
  \sum_{i=1}^{n_k}
  \ell\!\big(D(z_{x,i}^{[k]},N_k),y_i^{[k]}\big).
  \end{equation}
  This is the held-out task error of the deployed pathway itself.
\end{itemize}

\subsubsection*{Latent-Law Match}
\begin{itemize}
  \item \textbf{Graph/path matching calibration} (corresponds to $\varepsilon_c$).
  For each reference graph,
  \begin{align}
  \hat W_c^{[k]}
  &:=
  W_2\!\big(\hat\mu_x^{[k]},\hat\mu_s^{[k]}\big),\\
  \hat W_c
  &:=
  \left(\sum_k \pi_k(\hat W_c^{[k]})^2\right)^{1/2}.
  \end{align}
  Small $\hat W_c$ means the deployment and teacher pathways match graph by
  graph under the same demand law; this is a graph/path calibration
  constant rather than a cross-graph transfer gap. A conservative
  deterministic version may use $\max_k\hat W_c^{[k]}$. In the theorem,
  the transfer gap is represented by $\hat W_x^{[k]}$.
  \item \textbf{Semantic-reference Wasserstein proximity.}
  \begin{equation}
  \hat W_s^{[k]}
  :=
  W_2\!\big(\hat\mu_s^{[k]},\hat\mu_\star^s\big).
  \end{equation}
  Small $\hat W_s^{[k]}$ means graph $k$ remains in the same semantic
  task family as the shared reference law.
  \item \textbf{Test-graph deployment proximity} (test-time verifiable, corresponds to $\varepsilon_x$).
  \begin{equation}
  \hat W_x^{[k]}
  :=
  W_2\!\big(\hat\mu_x^{[k]},\hat\mu_\star^x\big).
  \end{equation}
  Small $\hat W_x^{[k]}$ directly upper-bounds the contribution of
  graph $k$ to transfer error. This requires only $\hat\mu_x^{[k]}$
  and $\hat\mu_\star^x$, both from path X, so path S is not needed at
  test time.
\end{itemize}

\subsubsection*{Cross-Modal Alignment}
\begin{itemize}
  \item \textbf{Alignment error.}
  \begin{equation}
  \tau_a^{[k]}
  :=
  \frac{1}{n_k}
  \sum_{i=1}^{n_k}
  \|z_{x,i}^{[k]}-z_{s,i}^{[k]}\|_2.
  \end{equation}
  This gives a paired-sample notion of teacher/deployment consistency that is stronger than law matching alone.
\end{itemize}

The core deployment certificate implied by Theorem~\ref{thm:law_sufficient} is the bound itself. The auxiliary teacher-access and alignment quantities above are useful in controlled experiments, but the deployment-time certificate only needs $\hat W_c$, $\hat W_x^{[k]}$, and the constants $\hat\varepsilon_\star^s,\hat K$.

Beyond a binary pass/fail, the per-instance quantities $\hat W_x^{[k]}$ and $\hat W_s^{[k]}$ serve as targeted-expansion signals: instances that fail the certificate identify themselves as the right candidates for fine-tuning the deployment encoder, expanding the architecture's transfer envelope without full retraining. The empirical realization of this certificate-guided active-coverage loop is reported in \S\ref{sec:exp_powergrid}.

\begin{proposition}[Why path fidelity is meaningful]
\label{prop:pathverify}
For every graph $k$, using the per-graph loss-composition
$F_D^{[k]}(z)=\ell(D(z,N_k),y_k(z))$,
\begin{equation}
R_{\hat\mu_s^{[k]}}^{[k]}(D)=\hat\varepsilon_s^{[k]}.
\end{equation}
The quantity $R_{\hat\mu_x^{[k]}}^{[k]}(D)$ denotes the latent-law
deployment risk associated with the graph-$k$ target map $y_k$.
If $F_D^{[k]}$ is $K$-Lipschitz on the latent region supporting
$\hat\mu_s^{[k]}$ and $\hat\mu_\star^s$, then
\begin{equation}
\big|R_{\hat\mu_s^{[k]}}^{[k]}(D)-R_{\hat\mu_\star^s}^{[k]}(D)\big|
\le
K \hat W_s^{[k]}.
\end{equation}
Consequently,
\begin{equation}
R_{\hat\mu_\star^s}^{[k]}(D)
\le
\hat\varepsilon_s^{[k]} + K \hat W_s^{[k]}.
\label{eq:verify_ref_bound}
\end{equation}
\end{proposition}
\noindent\emph{Proof: see Appendix~\ref{app:verify_proofs}.}

\begin{proposition}[Why latent-law proximity is meaningful]
\label{prop:lawverify}
Assume $F_D^{[k]}$ is $K$-Lipschitz on the latent region supporting
$\hat\mu_x^{[k]}$, $\hat\mu_\star^x$, and $\hat\mu_\star^s$.
Then, for every graph $k$,
\begin{equation}
R_{\hat\mu_x^{[k]}}^{[k]}(D)
\le
\hat\varepsilon_s^{[k]} + K(\hat W_s^{[k]}+\hat W_c + \hat W_x^{[k]}).
\label{eq:verify_dep_bound}
\end{equation}
\end{proposition}
\noindent\emph{Proof: see Appendix~\ref{app:verify_proofs}.}

\begin{proposition}[Why alignment error controls deployment error]
\label{prop:alignverify}
Assume the decoder $D(\cdot,N_k)$ is $L_D$-Lipschitz on the latent region of graph $k$, and assume the loss is $L_\ell$-Lipschitz in its first argument, i.e.,
\begin{equation}
\big|\ell(\hat y,y)-\ell(\hat y',y)\big|
\le
L_\ell \|\hat y-\hat y'\|
\qquad
\forall \hat y,\hat y',y.
\end{equation}
Then
\begin{equation}
\hat\varepsilon_x^{[k]}
\le
\hat\varepsilon_s^{[k]} + L_\ell L_D \tau_a^{[k]}.
\label{eq:verify_align_bound}
\end{equation}
\end{proposition}
\noindent\emph{Proof: see Appendix~\ref{app:verify_proofs}.}

\begin{theorem}[What Algorithm~\ref{alg:verify} certifies]
\label{thm:verifycert}
Assume $F_D^{[k]}$ is $K$-Lipschitz and that
$\hat\varepsilon_\star^s$ upper-bounds
$R_{\hat\mu_\star^s}^{[k]}(D)$ on the evaluated graphs. Let
$\widehat{p}_\delta$, $\hat\tau_\alpha$, and $\hat B_\alpha$ be the
outputs of Algorithm~\ref{alg:verify}. Then:
\begin{enumerate}
  \item \textbf{Empirical success rate.}
  $\widehat{p}_\delta$ is the exact fraction of test graphs on which
  the deployment error satisfies $\hat\varepsilon_x^{[k]}\le\delta$.
  \item \textbf{Quantile transfer bound.}
  For any graph $k$ with $\hat W_x^{[k]}\le\hat\tau_\alpha$,
  \begin{equation}
	  R_{\hat\mu_x^{[k]}}^{[k]}(D)
	  \le \hat B_\alpha
	  :=
	  \hat\varepsilon_\star^s+\hat K(\hat W_c+\hat\tau_\alpha).
  \end{equation}
  By Corollary~\ref{cor:prob_transfer}, a fraction $\ge\alpha$ of
  test graphs satisfy this bound ($\hat W_c$ is a graph/path calibration
  constant). This probabilistic statement deliberately avoids treating
  the finite test-set maximum as a future uniform bound.
\end{enumerate}
\end{theorem}
\noindent\emph{Proof: see Appendix~\ref{app:verify_proofs}.}

Algorithm~\ref{alg:verify} implements this protocol.

\begin{algorithm}
\caption{Probabilistic Bound Verification}
\label{alg:verify}
\begin{algorithmic}[1]
\Require Reference paired latent sets $\{(Z_s^{[k]},Z_x^{[k]})\}_{k\in\mathcal{D}_\mathrm{ref}}$,
         deployment test sets $\{(Z_x^{[k]},Y^{[k]})\}_{k=1}^{K_\text{test}}$,
         training reference laws $\hat\mu_\star^s$ and $\hat\mu_\star^x$,
         semantic-fidelity baseline $\hat\varepsilon_\star^s$,
         Lipschitz level $\hat K$,
         error tolerance $\delta$,
         confidence level $\alpha\in(0,1)$
\State \textbf{Graph/path matching calibration}:
       compute $\hat W_c^{[k]} \leftarrow W_2(\hat\mu_x^{[k]},\hat\mu_s^{[k]})$
       on each reference graph and set
       $\hat W_c \leftarrow (\sum_k\pi_k(\hat W_c^{[k]})^2)^{1/2}$
\State Initialize per-graph record $\mathcal{R}\leftarrow\{\}$
\For{each graph $k$}
  \State Compute $\hat\varepsilon_x^{[k]}$ and
         $\hat W_x^{[k]}{:=}W_2(\hat\mu_x^{[k]},\hat\mu_\star^x)$
  \State Record $\mathcal{R}[k] \leftarrow
         \bigl(\hat\varepsilon_x^{[k]},\,\hat W_x^{[k]}\bigr)$
\EndFor
\State \textbf{Empirical success rate:}
       $\;\widehat{p}_\delta \leftarrow
       \frac{1}{K_\text{test}}
       \sum_{k} \mathbf{1}\!\left[\hat\varepsilon_x^{[k]}\le\delta\right]$
\State \textbf{Deployment transfer-gap quantile:}
       $\;\hat\tau_\alpha \leftarrow
       \alpha\text{-quantile of }
       \bigl\{\hat W_x^{[k]}\bigr\}_{k=1}^{K_\text{test}}$
\State \textbf{Quantile bound:}
       $\;\hat B_\alpha \leftarrow \hat\varepsilon_\star^s+\hat K(\hat W_c+\hat\tau_\alpha)$
\State \textbf{Empirical bound pass rate:}
       $\;\widehat{p}_{\mathrm{bd}}\leftarrow
       \frac{1}{K_\text{test}}\sum_k
       \mathbf{1}\!\left[\hat\varepsilon_x^{[k]}
       \le \hat\varepsilon_\star^s+\hat K(\hat W_c+\hat W_x^{[k]})\right]$
\State \Return $\widehat{p}_\delta,\;\hat\tau_\alpha,\;\hat W_c,\;
         \hat B_\alpha,\;\widehat{p}_{\mathrm{bd}}$
\end{algorithmic}
\end{algorithm}

Algorithm~\ref{alg:verify} returns the quantities used in the reported
bound and pass rate. In controlled studies where path S is available,
the same framework can additionally report teacher-access quantities
such as $\hat W_s^{[k]}$, but they are not part of the deployment-only
certificate.

\section{Experiments}
\label{sec:exp}

The empirical study has two parts. Section~\ref{sec:exp_graphon} presents a controlled simulation on graphon-generated graphs where the ground-truth structure of the transfer problem is known: it validates the sufficient-condition bound of Theorem~\ref{thm:law_sufficient} and characterizes how the certificate quantities vary as the test family shifts away from the training distribution. Section~\ref{sec:exp_powergrid} instantiates the framework on power-system state estimation across varying network topologies, where the graphon-level analysis of Section~\ref{sec:exp_graphon} informs our interpretation of the empirical transfer behavior.

\subsection{Controlled Graphon Simulation: Bound Validation}
\label{sec:exp_graphon}

\subsubsection{Setup}
We construct a controlled transfer problem using a spatial RBF graphon
\begin{equation}
P(\text{edge}\;i,j)
=
\exp\!\left(-\frac{\|u_i - u_j\|^2}{2\sigma^2}\right),
\qquad
u_i \sim \mathrm{Unif}([0,1]^2),
\end{equation}
with $\sigma = 0.16$. Signals are smooth spatial fields obtained by diffusing a random source over the sampled graph for five steps with spatial coupling $\gamma = 0.72$; they represent a clean synthetic analogue of physical fields on irregular networks. The teacher path $\Ps$ observes the full diffused field on all $N$ nodes; the deployment path $\Px$ observes only a random $30\%$ mask. Both paths encode into a shared $d=32$-dimensional latent space; the decoder uses $h=128$ hidden units and $K=4$ Chebyshev orders with $K_\text{eig}=2$ Laplacian eigenvectors as positional descriptors.

We train on $5{,}000$ graphs drawn from the \textbf{spatial} family with sizes $N\in\{40,80,160,250\}$ for 100 epochs (Adam, $\text{lr}=10^{-3}$, batch size 64). We then evaluate on 50 held-out test graphs per family across sizes $N\in\{40,80,160,250,400\}$ for three test families of increasing structural distance from the training distribution:
\begin{itemize}
  \item \textbf{spatial}: same graphon law and $\sigma$ as training (in-family);
  \item \textbf{spatial\_wide}: same spatial graphon with wider connectivity ($\sigma$ increased), yielding smoother graphs;
  \item \textbf{sbm\_spatial}: a community-spatial hybrid where community structure modulates the spatial affinity, representing a qualitatively different generative mechanism.
\end{itemize}
All test graphs are completely unseen during training.

\noindent\textbf{Baselines.}
We compare against five graph neural network baselines implemented via PyTorch Geometric: GCN~\citep{kipf2017semi}, GAT~\citep{velickovic2018graph}, GraphSAGE~\citep{hamilton2017inductive}, ChebNet~\citep{defferrard2016convolutional}, and Graph Transformer~\citep{shi2021masked}. We additionally include a coordinate-only MLP that takes node positional descriptors as input without message passing. All baselines are trained and evaluated on the same graph families and sizes using identical train/test splits.

\subsubsection{Certificate Computation}
We apply Algorithm~\ref{alg:verify} to all $150$ test graphs. The semantic reference law $\hat\mu_\star^s$ and the deployment reference law $\hat\mu_\star^x$ are pooled from 16 encoded snapshots per training graph. The graph/path matching calibration is computed graph by graph,
\begin{equation}
\hat W_c^{[k]} := W_2(\hat\mu_x^{[k]},\hat\mu_s^{[k]}),
\qquad
\hat W_c := \left(\sum_k\pi_k(\hat W_c^{[k]})^2\right)^{1/2}.
\end{equation}
For each test graph $k$ we compute the per-graph certificate quantities: deployment risk $\hat\varepsilon_x^{[k]}$, teacher-path risk $\hat\varepsilon_s^{[k]}$, deployment proximity $\hat W_x^{[k]} = W_2(\hat\mu_x^{[k]},\hat\mu_\star^x)$, semantic drift diagnostic $\hat W_s^{[k]} = W_2(\hat\mu_s^{[k]},\hat\mu_\star^s)$, and paired alignment error $\tau_a^{[k]}$. All $W_2$ distances are estimated by sliced Wasserstein with 128 random projections. The deployment-only certificate implied by Theorem~\ref{thm:law_sufficient} uses $\hat W_c+\hat W_x^{[k]}$. Because this controlled graphon study has access to path S on test graphs, we also report the teacher-access validation bound obtained by replacing the training constant $\hat W_c$ with the observed semantic drift $\hat W_s^{[k]}$; this is the bound visualised in Figures~\ref{fig:bound_validation}--\ref{fig:risk_vs_size}. Its Lipschitz constant is estimated empirically as the global $95$th percentile of the per-graph required constant
\begin{equation}
\hat K^{[k]}
:=
\max\!\left(0,\,
\frac{\hat\varepsilon_x^{[k]} - \hat\varepsilon_\star^s}
     {\hat W_s^{[k]} + \hat W_x^{[k]}}
\right),
\end{equation}
giving $\hat K_{95} = 1.10$ pooled across all $150$ test graphs and sizes. A test graph \emph{passes} the teacher-access validation bound if $\hat\varepsilon_x^{[k]} \le \hat\varepsilon_\star^s + \hat K_{95}(\hat W_s^{[k]} + \hat W_x^{[k]})$.

\subsubsection{Results}
Table~\ref{tab:graphon_cert} reports per-family certificate statistics averaged over all test graphs and sizes. Figure~\ref{fig:graphon_benchmarks} summarizes the corresponding reconstruction performance against GCN~\citep{kipf2017semi}, GAT~\citep{velickovic2018graph}, GraphSAGE~\citep{hamilton2017inductive}, ChebNet~\citep{defferrard2016convolutional}, Graph Transformer~\citep{shi2021masked}, and a coordinate MLP baseline, while Figures~\ref{fig:graphon_certificate}--\ref{fig:risk_vs_size} visualise the certificate quantities and underlying per-graph data. Three results stand out.

\emph{The validation bound is non-vacuous and holds at high rates.}
With $\hat K_{95} = 1.10$ and $\hat\varepsilon_\star^s = 0.0171$, the per-family teacher-access validation bound evaluates to $0.032$ (spatial), $0.033$ (spatial\_wide), and $0.038$ (sbm\_spatial), against mean deployment risks of $0.0144$, $0.0084$, and $0.0112$ respectively. Pass rates are $87\%$, $100\%$, and $98\%$ for $N\ge80$, and all families reach $100\%$ pass rate for $N\ge160$. The failures are concentrated at the smallest size ($N=80$), where the graph is too sparse for the privileged signal to diffuse faithfully.

\emph{Wasserstein distances predict bound tightness.}
The spatial\_wide family achieves the lowest mean $R_k$ ($0.0084$) and the highest pass rate ($100\%$) because wider connectivity yields smoother latent codes that remain close to the deployment reference: $\hat W_x = 0.0052$, comparable to the in-family $0.0051$. The sbm\_spatial family, despite a qualitatively different generative mechanism, maintains $\hat W_s = 0.0129$ and $\hat W_x = 0.0057$, close enough to the reference laws to achieve $98\%$ pass rate. This confirms the theoretical claim of Theorem~\ref{thm:law_sufficient}: what matters for transfer is latent-law proximity to the learned reference distributions, not the structural identity of the generating process.

\emph{Risk decreases monotonically with graph size; certificate distances are size-stable.}
For the spatial family, mean $R_k$ decreases rapidly with $N$, while $\hat W_s$ and $\hat W_x$ vary only mildly across sizes. This decoupling is structurally expected: $R_k$ depends on the quality of the sparse reconstruction, which improves as the graph provides more averaging, whereas the Wasserstein distances are distribution-level quantities insensitive to individual-graph noise.

\begin{table*}[t]
\centering
\caption{Transfer certificate statistics per test family ($50$ graphs each, sizes $N\in\{80,160,250,400,600,800,1000\}$). $\hat\varepsilon_\star^s = 0.0171$ and $\hat K_{95} = 1.10$ are global constants estimated from training data. The \emph{Bound} column reports the controlled teacher-access validation bound $\hat\varepsilon_\star^s + \hat K_{95}(\bar{\hat W}_s + \bar{\hat W}_x)$; the deployment-only theorem uses the same form with $\bar{\hat W}_s$ replaced by the training constant $\hat W_c$. A graph passes if its individual deployment risk lies below its individual validation bound.}
\label{tab:graphon_cert}
\small
\setlength{\tabcolsep}{5pt}
\begin{tabular}{lccccccc}
\toprule
\textbf{Family} & $\bar{R}_k$ & $\bar{\hat W}_s$ & $\bar{\hat W}_x$ & $\bar\tau_a$ & \textbf{Bound} & \textbf{Pass} ($N{\ge}80$) & \textbf{Pass} ($N{\ge}160$) \\
\midrule
spatial      & 0.0144 & 0.0085 & 0.0051 & 0.036 & 0.032 &  87\% & 100\% \\
spatial\_wide & 0.0084 & 0.0093 & 0.0052 & 0.036 & 0.033 & 100\% & 100\% \\
sbm\_spatial  & 0.0112 & 0.0129 & 0.0057 & 0.041 & 0.038 &  98\% & 100\% \\
\bottomrule
\end{tabular}
\end{table*}

\begin{figure*}[t]
  \centering
  \begin{minipage}[t]{0.32\linewidth}
    \centering
    \includegraphics[width=\linewidth]{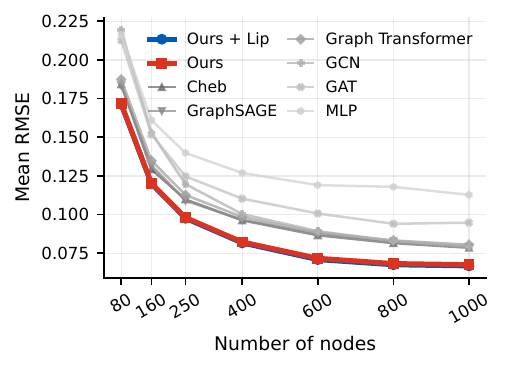}\\[-1mm]
    \textbf{(a)} Size-wise RMSE
  \end{minipage}\hfill
  \begin{minipage}[t]{0.32\linewidth}
    \centering
    \includegraphics[width=\linewidth]{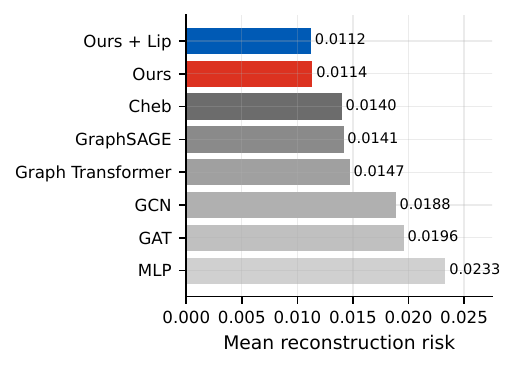}\\[-1mm]
    \textbf{(b)} Overall ranking
  \end{minipage}\hfill
  \begin{minipage}[t]{0.32\linewidth}
    \centering
    \includegraphics[width=\linewidth]{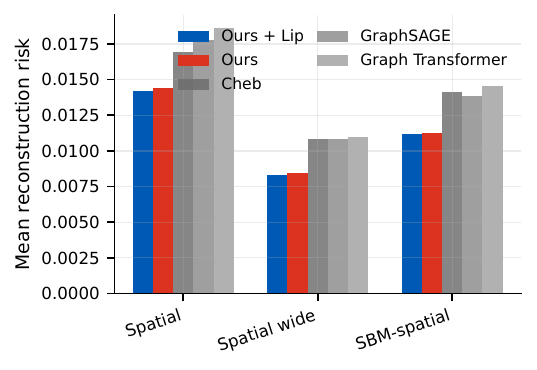}\\[-1mm]
    \textbf{(c)} Family-wise risk
  \end{minipage}
  \caption{Graphon benchmark reconstruction performance against GCN~\citep{kipf2017semi}, GAT~\citep{velickovic2018graph}, GraphSAGE~\citep{hamilton2017inductive}, ChebNet~\citep{defferrard2016convolutional}, Graph Transformer~\citep{shi2021masked}, and a coordinate MLP baseline.
  \textbf{(a)} Mean RMSE decreases with graph size, with the proposed
  transfer model matching or improving on the Lipschitz-regularized variant
  across the tested node counts.
  \textbf{(b)} Overall reconstruction risk ranks the proposed transfer models
  ahead of graph neural and coordinate-only baselines.
  \textbf{(c)} Family-wise reconstruction risk shows the same advantage on
  the spatial, spatial\_wide, and sbm\_spatial test families used in
  Table~\ref{tab:graphon_cert}.}
  \label{fig:graphon_benchmarks}
\end{figure*}

\begin{figure}[t]
  \centering
  \includegraphics[width=0.5\linewidth]{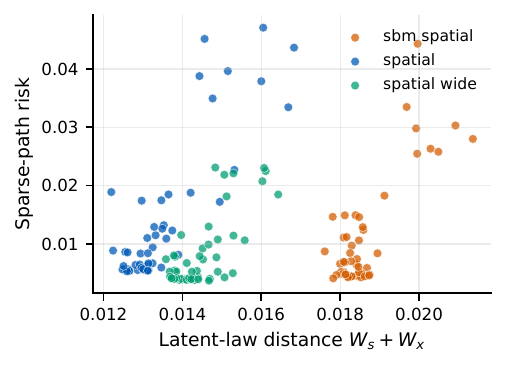}
  \caption{Graphon certificate scatter for the proposed model.
  Each point is one held-out graph; the horizontal axis is the latent-law
  distance $\hat W_s+\hat W_x$ and the vertical axis is the sparse-path
  deployment risk. The family separation mirrors Table~\ref{tab:graphon_cert}
  and shows that the shifted sbm\_spatial family remains within the
  low-risk certificate regime.}
  \label{fig:graphon_certificate}
\end{figure}

\subsubsection{Bound Analysis}
The key quantity governing tightness is $K$. We estimate $K$ in two complementary ways. First, the empirical required-$K$ estimator above gives $\hat K_{95} = 1.10$, a pure output-space quantity requiring no internal network access. Second, following Lemma~\ref{lem:primitive_lipschitz}, we estimate the loss-composition Lipschitz constant $K = L_\ell(L_D + L_y)$ from neural Jacobians: for each test snapshot we compute the spectral norm of the decoder Jacobian $\partial D / \partial z$ at $z = E_x(x)$. The resulting empirical $95$th-percentile loss-composition constant is $\hat K^\text{Jac}_{95} = 0.83$, tighter than $\hat K_{95} = 1.10$ but substantially harder to estimate in closed form. The Jacobian-based estimate suggests the actual $K$ is well below $1.10$; the excess arises because $\hat K_{95}$ absorbs variance across graph sizes and random seeds rather than worst-case operator norms. Both estimates confirm the bound is non-vacuous, and the Jacobian estimate points toward future tightening by explicit Lipschitz regularization of the decoder.

\begin{figure}[t]
  \centering
  \includegraphics[width=\linewidth]{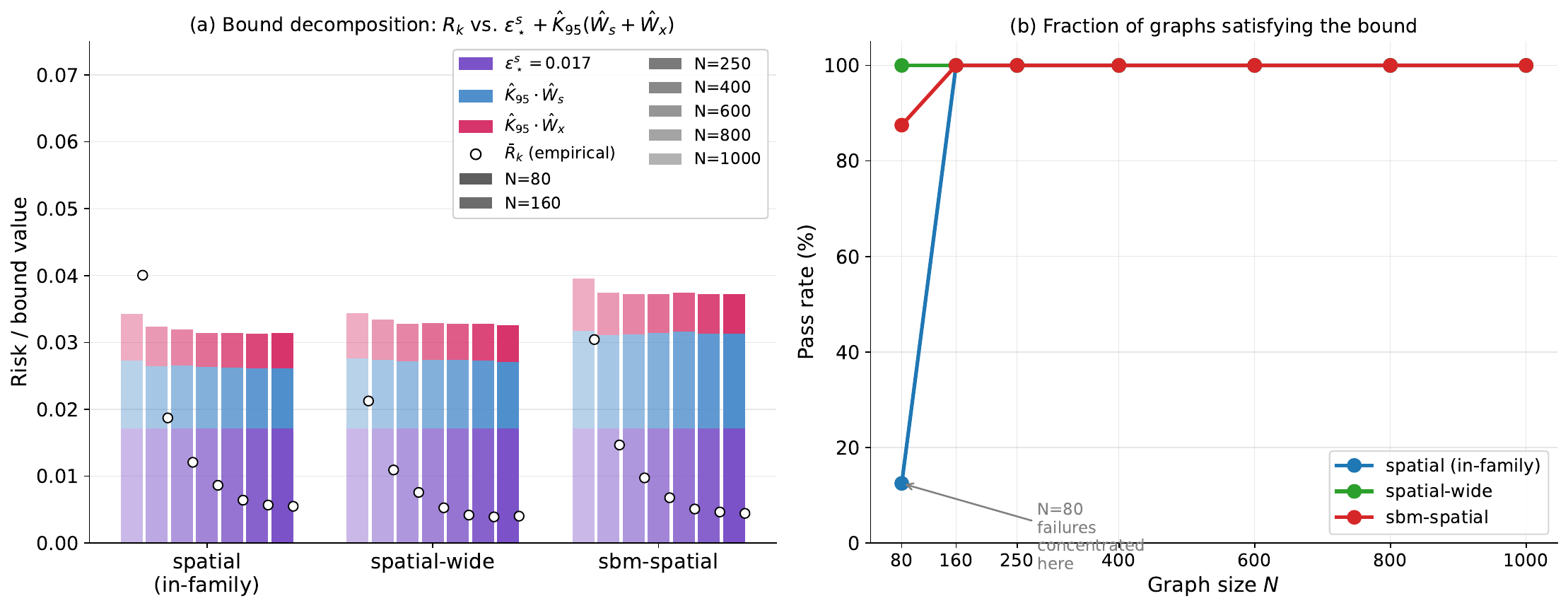}
  \caption{Empirical verification of the controlled teacher-access validation bound
  $R_k \le \varepsilon_\star^s + \hat K_{95}(\hat W_s+\hat W_x)$
  on the $150$-graph graphon benchmark.
  \textbf{(a)} Stacked bar decomposition of the bound into its three
  theorem-driven components: $\varepsilon_\star^s = 0.017$ (purple base),
  $\hat K_{95}\hat W_s$ (indigo), and $\hat K_{95}\hat W_x$ (pink), for each
  family across graph sizes $N\in\{80,160,250,400,600,800,1000\}$.
  White dots mark the mean empirical deployment risk $\bar R_k$; in every bar
  the dot lies strictly below the bar top, confirming the bound.
  The dominant component is $\varepsilon_\star^s$; the two Wasserstein terms
  are small, consistent add-ons.
  \textbf{(b)} Pass rate (fraction of individual graphs satisfying the
  bound) as a function of $N$.  All families reach $100\%$ at $N\ge 160$;
  the sole exceptions are at $N=80$ (spatial: $12\%$, sbm\_spatial: $88\%$),
  where the graph is too sparse for the privileged signal to diffuse faithfully,
  inflating $R_k$.}
  \label{fig:bound_validation}
\end{figure}

\begin{figure*}[t]
  \centering
  \includegraphics[width=\linewidth]{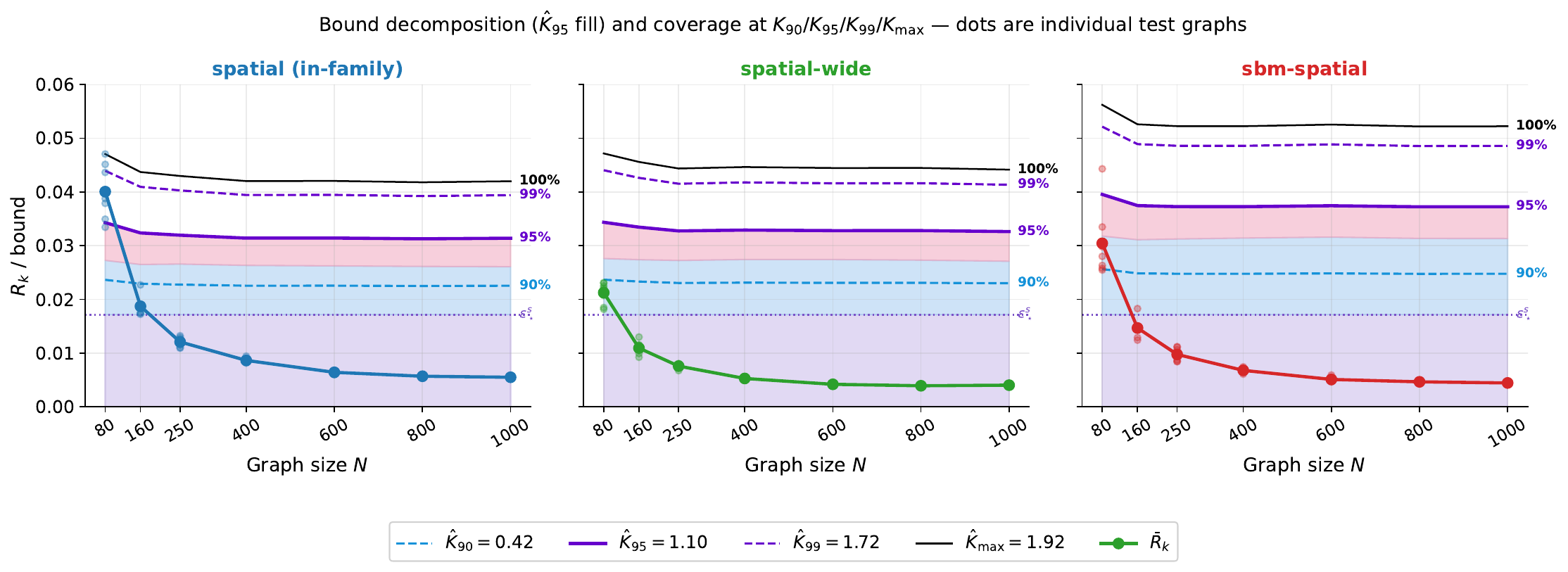}
  \caption{Deployment risk vs.\ graph size for the three test families.
  \textbf{Fill layers} (at $\hat K_{95}$ reference) decompose the bound into
  its three theorem-driven components: $\varepsilon_\star^s=0.017$ (lavender),
  $\hat K_{95}\hat W_s$ (blue), and $\hat K_{95}\hat W_x$ (pink).
  \textbf{Curves} show bounds at four $K$ percentile levels
  ($\hat K_{90}=0.42$, $\hat K_{95}=1.10$, $\hat K_{99}=1.72$,
  $\hat K_{\max}=1.92$); pass rates appear at the right of each curve.
  \textbf{Dots} are individual per-graph $R_k$ values ($10$ per size);
  the coloured line is the family mean $\bar R_k$.
  Two observations follow.
  First, $\bar R_k$ falls monotonically from ${\sim}0.09$ at $N=80$ to
  $<0.01$ at $N=1000$, while the bound layers remain nearly flat: $R_k$
  improves with graph size but $\hat W_s$, $\hat W_x$ are distribution-level
  quantities insensitive to individual-graph noise (\emph{decoupling}).
  Second, at $\hat K_{99}$ all families achieve $\ge 98\%$ pass rate; the
  test points concentrate well below the $\hat K_{90}$ curve, confirming that
  $\hat K_{95}$ is a conservative estimate of the true Lipschitz constant.}
  \label{fig:risk_vs_size}
\end{figure*}

\subsubsection{Effect of Observation Density on Reconstruction and Transfer Quality}
The deployment path $\Px$ observes only a randomly masked subset of nodes at each snapshot.
To assess how the mask rate $\rho$ (the fraction of nodes retained as measurements)affects both reconstruction fidelity and transfer certificate quality, we train four model variants with $\rho \in \{0.1, 0.2, 0.3, 0.4\}$, holding all other hyperparameters fixed.
Table~\ref{tab:mask_ablation} reports RMSE of the sparse-path reconstruction together with the per-family mean deployment risk $\bar R_k$ and the global $\hat K_{95}$ constant.

Increasing $\rho$ consistently tightens the certificate: $\hat K_{95}$ decreases from $1.90$ at $\rho=0.1$ to $1.10$ at $\rho=0.3$, and $\bar R_k$ falls across all three test families.
RMSE improves markedly from $\rho=0.1$ to $\rho=0.3$ (from $0.113$ to $0.096$) and continues to decrease to $0.095$ at $\rho=0.4$, while the certificate constant $\hat K_{95}$ is nearly identical between the two ($1.103$ vs.\ $1.130$), suggesting diminishing returns beyond $\rho=0.3$.
The improvement from denser observations is structurally expected: a higher mask rate provides more latent averaging at encoding time, pulling the deployment distribution $\hat\mu_x$ closer to the semantic reference $\hat\mu_\star^s$ and reducing the alignment error $\tau_a$.
All reported results in the preceding analysis use $\rho = 0.3$.

\begin{table}[t]
\centering
\caption{Ablation on observation density (mask rate $\rho$). $\text{RMSE}_x$ and $\text{RMSE}_s$ are the deployment-path and teacher-path reconstruction errors. $\bar R_k$ is the mean deployment risk per test family; $\hat K_{95}$ is the global Lipschitz constant estimated over all $150$ test graphs. Lower is better for all metrics.}
\label{tab:mask_ablation}
\small
\setlength{\tabcolsep}{4.5pt}
\begin{tabular}{cccccccc}
\toprule
$\rho$ & $\text{RMSE}_x$ & $\text{RMSE}_s$ & $\bar R_k^\text{spatial}$ & $\bar R_k^\text{spatial\_wide}$ & $\bar R_k^\text{sbm\_spatial}$ & $\hat K_{95}$ \\
\midrule
0.1 & 0.1129 & 0.0934 & 0.0211 & 0.0117 & 0.0173 & 1.899 \\
0.2 & 0.1143 & 0.1001 & 0.0185 & 0.0112 & 0.0150 & 1.228 \\
0.3 & 0.0964 & 0.1062 & 0.0144 & 0.0084 & 0.0112 & 1.103 \\
0.4 & 0.0951 & 0.0941 & 0.0142 & 0.0081 & 0.0110 & 1.130 \\
\bottomrule
\end{tabular}
\end{table}

\subsection{Power-System State Estimation Across Unseen Topologies}
\label{sec:exp_powergrid}

\subsubsection{Setup}
We evaluate on synthetic power distribution networks derived from the IEEE~33-bus benchmark with randomised line parameters and bus counts $N\in[14,38]$.
A base model is first trained on $1{,}000$ case33-derived graphs; test graphs ($100$ topologies, $N\in[15,37]$) are entirely unseen in structure, bus count, and admittance matrix.
The latent dimension is $d=128$; the decoder uses Chebyshev order $K_\text{cheb}=4$ with $K_\text{eig}=2$ Laplacian eigenvector features.
Training uses Adam ($\text{lr}=5\times10^{-6}$), batch size $32$ with $8$ graphs per batch, and consistency weight $\lambda_c=1$.

\subsubsection{Certificate-Guided Active Coverage Expansion}
A teacher-access screening pass, used only during active expansion when the teacher pathway is available, partitions any candidate graph pool into \emph{covered} graphs ($\hat W_s^{[k]}\le\tau_{ws}$) and \emph{uncovered} graphs ($\hat W_s^{[k]}>\tau_{ws}$).
An uncovered graph lies outside the latent geometry shaped by the current training distribution, and the bound $\hat\varepsilon_\star^s+\hat K(\hat W_s^{[k]}+\hat W_x^{[k]})$ is correspondingly wide for these topologies.
Coverage expansion uses the uncovered set as a targeted fine-tuning corpus: for each uncovered graph $k$, the deployment encoder $E_x$ is updated until $\hat W_s^{[k]}<\tau_{ws}$, after which $\hat\mu_\star^s$ is recomputed on the augmented set, yielding a tighter $\hat\varepsilon_\star^s$ and a sharper bound.
This closed loop, teacher-access screen, fine-tune on uncovered graphs, re-evaluate the bound, constitutes \emph{certificate-guided active learning}.
Concretely, after initial training, the teacher-access screen is applied to a pool of $10{,}000$ newly generated candidate graphs to identify topologies not yet covered by the learned latent geometry.
A graph is deemed \emph{uncovered} if $\hat W_s^{[k]} > \tau_{ws} = 0.111$, indicating that its teacher encoder output lies outside the training reference region.
Of the $10{,}000$ candidates, $7{,}501$ ($75\%$) are uncovered; coverage is nearly zero for $N\le18$ and rises monotonically to $100\%$ at $N\ge32$.
Notably, all $N=32$--$33$ graphs in the pool are \emph{covered}: the model trained on case33-derived topologies already generalises to the nominal 33-bus topology, and active expansion is needed precisely for the smaller novel graphs ($N\le21$) that lie outside the training distribution.
We select $2{,}000$ uncovered graphs for fine-tuning and run $20$ additional epochs with a Lipschitz-regularisation term ($\lambda_\text{lip}=10^{-4}$) on the decoder, which reduces the empirical semantic-fidelity reference $\hat\varepsilon_\star^s$ from $0.209$ to $0.112$ and tightens the Jacobian-based Lipschitz estimate to $\hat K_{95}^\text{Jac}=1.44$.

\subsubsection{State-Estimation Results}
Table~\ref{tab:rmse} reports zero-shot voltage-phasor RMSE on $100$ unseen test topologies at varying AMI sensor coverage $\rho_\text{AMI}$, compared to two topology-aware classical baselines.
Our estimator dominates at low coverage ($10\%$--$20\%$), where the teacher-shaped latent geometry compensates for the severely under-determined observation; at $30\%$--$50\%$ coverage, NR-WLS recovers the advantage by exploiting exact topology information, which is unavailable to our zero-shot estimator.
Inference time is constant at ${\sim}0.80$~ms regardless of graph size; NR-WLS scales from $18.8$~ms to $138.7$~ms.
Table~\ref{tab:pgcert} further breaks down RMSE and bound pass rate by bus-count group.
RMSE increases monotonically with graph size, from $0.00505$ at $N\in[15,19]$ to $0.01332$ at $N\in[30,37]$.
Bound pass rates are $100\%$ for both small groups and degrade to $86\%$ for $N\in[30,37]$, where larger deployment-path latent drift inflates $\hat W_x$ and raises $\hat K_{95}$ requirements.
Overall bound pass rate across all $100$ test graphs is $\mathbf{95\%}$ at the empirical $\hat K_{95}=3.70$ level.

\begin{table}[!t]
\centering
\caption{Zero-shot voltage-phasor RMSE on $100$ unseen test topologies at varying AMI sensor coverage. NR-WLS and DistFlow require exact topology and re-optimize per graph; they are oracle references, not transfer baselines.}
\label{tab:rmse}
\small
\setlength{\tabcolsep}{3.2pt}
\begin{tabular}{lcccccc}
\toprule
& \multicolumn{2}{c}{\textbf{Ours (Zero-Shot)}}
& \multicolumn{2}{c}{NR-WLS$^\ddagger$}
& \multicolumn{2}{c}{DistFlow$^\ddagger$} \\
\cmidrule(lr){2-3}\cmidrule(lr){4-5}\cmidrule(lr){6-7}
$\rho_\text{AMI}$ & RMSE & ms & RMSE & ms & RMSE & ms \\
\midrule
10\% & \textbf{0.01044} & 0.785 & 0.01544 & 18.81  & 0.04402 & 0.587 \\
20\% & \textbf{0.00942} & 0.793 & 0.01002 & 35.59  & 0.02848 & 0.708 \\
30\% & 0.00917          & 0.792 & \textbf{0.00815} & 83.73  & 0.02703 & 0.869 \\
40\% & 0.00852          & 0.810 & \textbf{0.00795} & 108.82 & 0.02603 & 0.970 \\
50\% & 0.00851          & 0.815 & \textbf{0.00767} & 138.68 & 0.02332 & 0.987 \\
\bottomrule
\end{tabular}
\\[2pt]
\scriptsize $^\ddagger$Not zero-shot transferable: requires exact topology, re-optimized per graph.
\end{table}

\begin{table}[!t]
\centering
\caption{RMSE and bound pass rate by bus-count group on $100$ test graphs (post active expansion). Bound pass: $\hat R_k \le \hat\varepsilon_\star^s + \hat K_{95}(\hat W_s+\hat W_x)$ with $\hat\varepsilon_\star^s=0.112$, $\hat K_{95}=3.70$ (empirical 95th-percentile Lipschitz level); $\hat W_c\approx 0.003$ is absorbed into $\hat\varepsilon_\star^s$.}
\label{tab:pgcert}
\small
\setlength{\tabcolsep}{7pt}
\begin{tabular}{lccc}
\toprule
\textbf{Bus range} & $n$ & \textbf{Mean RMSE} & \textbf{Bound pass} \\
\midrule
$N\in[15,19]$ & 21 & 0.00505 & 100\% \\
$N\in[20,24]$ & 31 & 0.00679 & 100\% \\
$N\in[25,29]$ & 20 & 0.00917 &  95\% \\
$N\in[30,37]$ & 28 & 0.01332 &  86\% \\
\midrule
All           & 100 & 0.00873 & \textbf{95\%} \\
\bottomrule
\end{tabular}
\end{table}

\begin{figure*}[t]
  \centering
  \includegraphics[width=\linewidth]{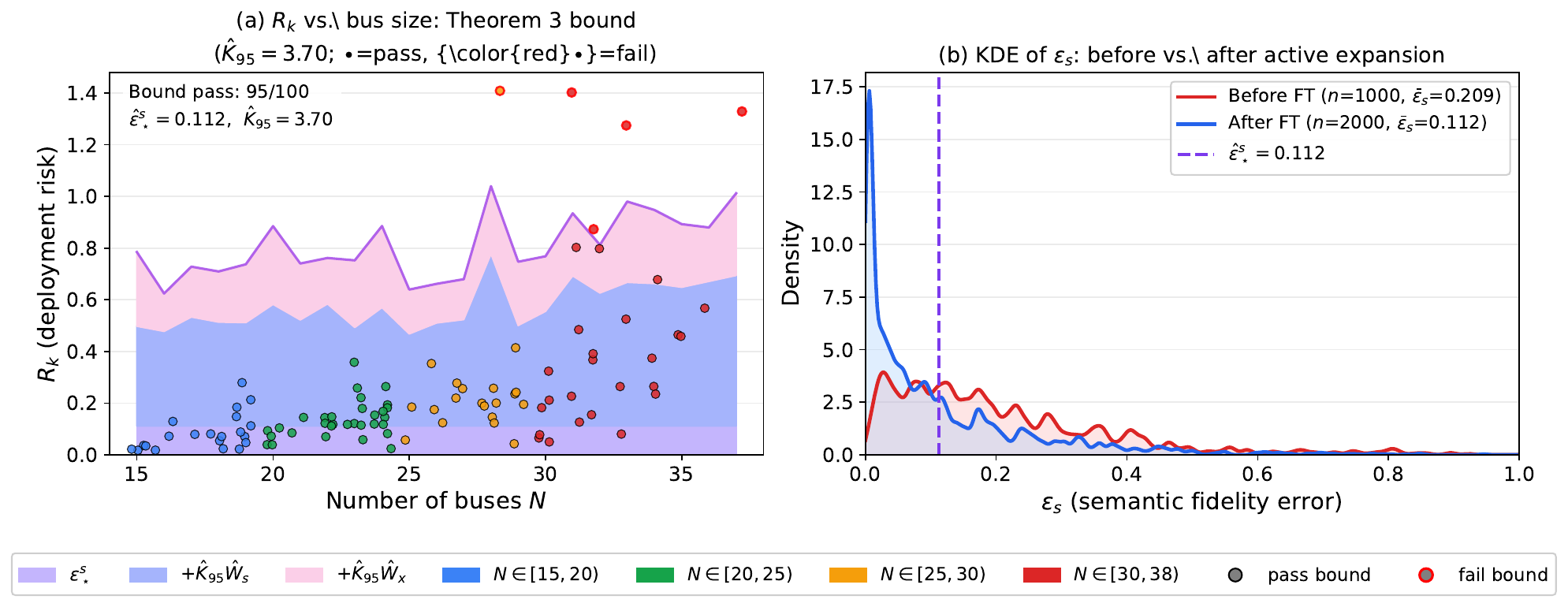}
  \caption{\textbf{(a)} $R_k$ vs.\ bus size $N$ for all $100$ test graphs.
  Stacked fill layers decompose the Theorem~\ref{thm:law_sufficient} bound:
  lavender $=\varepsilon_\star^s$, blue $=+\hat K_{95}\hat W_s$, pink $=+\hat K_{95}\hat W_x$
  (per-$N$ worst-case $\hat W_s,\hat W_x$; $\hat K_{95}=3.70$).
  Blue dots pass the $\hat K_{95}$ bound; red dots fail; overall bound pass rate is $95\%$.
  \textbf{(b)} KDE of the semantic fidelity error $\varepsilon_s$ on the reference set
  before (red, $n=1000$, $\bar\varepsilon_s=0.209$) and after (blue, $n=2000$, $\bar\varepsilon_s=0.112$)
  active coverage expansion.  The dashed line marks the tightened $\hat\varepsilon_\star^s=0.112$.}
  \label{fig:powergrid_results}
\end{figure*}

\subsubsection{Certificate Diagnostics}
Figure~\ref{fig:powergrid_results}a plots $R_k$ versus bus size for all $100$ test graphs with the stacked bound layers from Theorem~\ref{thm:law_sufficient}, mirroring Figure~\ref{fig:risk_vs_size}.
$95\%$ of graphs satisfy the $\hat K_{95}=3.70$ bound, consistent with the graphon experiment.
The dominant factor limiting bound tightness for large graphs ($N\ge28$) is $\hat W_x$: deployment-path latent codes drift further from $\hat\mu_\star^x$ as topology grows, inflating the $\hat K_{95}\hat W_x$ layer.
The graph/path matching calibration $\hat W_c = 0.003$ is negligible, so the bound reduces effectively to $\hat\varepsilon_\star^s + \hat K_{95}(\hat W_s^{[k]}+\hat W_x^{[k]})$ per instance, where $\hat W_s^{[k]}$ captures the residual semantic drift of each test graph from the post-expansion reference law.
Figure~\ref{fig:powergrid_results}b confirms that active expansion tightened $\hat\varepsilon_\star^s$ from $0.209$ to $0.112$, shifting the reference distribution leftward and narrowing the semantic fidelity spread.

\section{Conclusion}
\label{sec:conclusion}

This paper proposed an asymmetric two-pathway architecture for inference
and control problems in which the physics is expensive at deployment. A
teacher encoder $\Ps$ consumes a privileged dense signal during training
and shapes a latent geometry anchored by operator-polynomial features
whose coefficients are stable under spectral perturbation; a student
encoder $\Px$ is trained to reach the same latent geometry from sparse
field data and operator descriptors, and serves as the sole deployment
pathway. The pointwise transfer-compatibility criterion of
\S\ref{sec:geom_pointwise} answers when cross-instance transfer is
well-posed; the latent-law conditions of \S\ref{sec:geom_law} convert
geometric proximity into bounded task risk via Wasserstein stability,
sufficient and near-necessary; and the verification protocol of
\S\ref{sec:verify} translates these into a finite-sample per-instance
transfer certificate that requires only deployment-side data. The
certificate is not merely diagnostic: it identifies instances in a
candidate pool that fall outside the architecture's transfer envelope,
guiding targeted fine-tuning that expands latent coverage without full
retraining.

On a power-system state-estimation benchmark over $100$ unseen
distribution-network topologies, certificate-guided active expansion
reduces the semantic-fidelity reference $\hat\varepsilon_\star^s$ from
$0.209$ to $0.112$, achieves a $95\%$ bound pass rate at
$\hat K_{95}=3.70$ (rising to $100\%$ at $\hat K_{\max}=5.80$), and
yields mean voltage-phasor RMSE of $0.00873$. Inference is constant at
${\sim}0.8$~ms regardless of topology size, against $18.8$--$138.7$~ms
for a topology-aware Newton-Raphson reference that requires exact
topology. A complementary controlled graphon study confirms that the
bound is non-vacuous and that what governs transfer is Wasserstein
proximity to the deployment reference law, not structural identity of
the generating process: even a hybrid community-spatial test family
with a qualitatively different generative mechanism achieves $98\%$
pass rate at $N\ge 80$ and $100\%$ at $N\ge 160$. Together with the
geometric theory, these results say that the foundations of
cross-instance transfer are geometric: transfer is valid when
operator variation changes coordinates but does not change the latent
task.

Several extensions are natural. The Bayesian reading of
\S\ref{sec:bayesian} treats the regime $\theta$ as effectively
revealed by the teacher at training; at deployment $\theta$ is
generally unknown, and the genuinely Bayesian object is a soft E-step
,  a posterior over $\theta$ given $\mathbf{x}$ with the prediction
recovered as an integral over within-regime experts. This
mixture-of-experts deployment view, familiar from switching-state-space
models, time-varying channel estimation, and load-profile-conditioned
inference, is the next step we plan to develop. A second direction is
cross-domain extension beyond graphs: the framework applies wherever a
system admits a well-defined operator with a reportable spectrum,
including discretized Green's functions, channel matrices, and
state-transition Jacobians, and the natural follow-up paper instantiates
it on near-field array processing and target tracking. A third direction
is tightening the finite-sample bound via explicit Lipschitz
regularization of the decoder, which the Jacobian-based estimate
$\hat K_{95}^\text{Jac}=0.83$ on the controlled graphon study suggests
is feasible. A fourth is scaling the active-expansion pipeline to
larger transmission-level networks and extending the framework to
directed and time-varying graphs.

\appendix
\section{Proofs of Geometric Transfer Theorems}
\label{app:geom_proofs}

This appendix provides the full proofs of Theorems~\ref{thm:compat} and \ref{thm:transferbound} stated in Section~\ref{sec:geometry}. The first establishes that transfer-compatible graph families share a common latent coordinate system related by homeomorphisms; the second converts this structural statement into a finite sample error bound on the deployment pathway.

\subsection{Proof of Theorem~\ref{thm:compat}}

\begin{proof}
Fix any two graphs $k,k'$ in the transfer-compatible family, and let $U\subseteq\Mz$ be the shared connected latent set given by Definition~\ref{def:compat}. 
By assumption, the semantic encoders $E_s^{[k]}:\Ms^{[k]}\to U$
and $E_s^{[k']}:\Ms^{[k']}\to U$ are homeomorphisms onto $U$.

Define the coordinate map
\begin{equation}
\Psi^{[k\to k']}:\Ms^{[k]}\to\Ms^{[k']},\qquad
\Psi^{[k\to k']}:=\big(E_s^{[k']}\big)^{-1}\circ E_s^{[k]}.
\end{equation}
As the composition of two homeomorphisms, $\Psi^{[k\to k']}$ is itself
a homeomorphism between $\Ms^{[k]}$ and $\Ms^{[k']}$.
By construction, for every $s\in\Ms^{[k]}$,
\begin{equation}
E_s^{[k']}\!\big(\Psi^{[k\to k']}(s)\big)=E_s^{[k]}(s),
\end{equation}
so the two semantic pathways place corresponding signals at the same latent coordinate in $U$.
Now apply the shared-decoder property of Definition~\ref{def:compat}. For every $s\in\Ms^{[k]}$,
\begin{equation}
T^{[k]}(s)=D\!\big(E_s^{[k]}(s),N_k\big),
\end{equation}
and for $s':=\Psi^{[k\to k']}(s)\in\Ms^{[k']}$,
\begin{equation}
T^{[k']}(s')=D\!\big(E_s^{[k']}(s'),N_{k'}\big)=D\!\big(E_s^{[k]}(s),N_{k'}\big).
\end{equation}
Hence $T^{[k]}$ and $T^{[k']}\circ\Psi^{[k\to k']}$ differ only in the decoder's size argument $N$, and share the identical latent argument $E_s^{[k]}(s)\in U$. The latent task coordinate is therefore invariant across graphs in the family, while topology variation is fully absorbed into the homeomorphism $\Psi^{[k\to k']}$ and the size input of the shared decoder. This proves the claim.
\end{proof}

\subsection{Proof of Theorem~\ref{thm:transferbound}}

\begin{proof}
Fix any paired sample $(x^{[k]},s^{[k]})$ with $s^{[k]}\in\Ms^{[k]}$. By the triangle inequality applied to the norm on the decoder's output space,
\begin{equation}
\begin{aligned}
\big\|D(E_x^{[k]}(x^{[k]}),N_k)-T^{[k]}(s^{[k]})\big\|
&\le \big\|D(E_x^{[k]}(x^{[k]}),N_k)-D(E_s^{[k]}(s^{[k]}),N_k)\big\| \\
&\quad +\big\|D(E_s^{[k]}(s^{[k]}),N_k)-T^{[k]}(s^{[k]})\big\|.
\end{aligned}
\label{eq:geom_proof_triangle}
\end{equation}
The second term is the teacher-pathway reconstruction residual,
which is bounded by $\varepsilon_T$ via the
$\varepsilon_T$-approximate compatibility hypothesis of
Theorem~\ref{thm:transferbound}:
\begin{equation}
\big\|D(E_s^{[k]}(s^{[k]}),N_k)-T^{[k]}(s^{[k]})\big\|\le \varepsilon_T.
\label{eq:approx_compat_use}
\end{equation}

For the first term, invoke the $L_D$-Lipschitz property of $D(\cdot,N_k)$ on the shared latent set $U\supseteq\{E_x^{[k]}(x^{[k]}),E_s^{[k]}(s^{[k]})\}$:
\begin{equation}
\big\|D(E_x^{[k]}(x^{[k]}),N_k)-D(E_s^{[k]}(s^{[k]}),N_k)\big\|
\le L_D\big\|E_x^{[k]}(x^{[k]})-E_s^{[k]}(s^{[k]})\big\|.
\label{eq:geom_proof_lip}
\end{equation}
Substituting \eqref{eq:geom_proof_lip} and \eqref{eq:approx_compat_use}
into \eqref{eq:geom_proof_triangle} yields
\begin{equation}
\big\|D(E_x^{[k]}(x^{[k]}),N_k)-T^{[k]}(s^{[k]})\big\|
\le L_D\big\|E_x^{[k]}(x^{[k]})-E_s^{[k]}(s^{[k]})\big\|+\varepsilon_T,
\end{equation}
which is exactly \eqref{eq:transfer-bound}.


For the extension to unseen graphs, observe that the argument used
only (i) the triangle inequality, (ii) the uniform Lipschitz
property of $D(\cdot,N)$ on $U$, (iii) the fact that both
$E_x^{[k]}(x^{[k]})$ and $E_s^{[k]}(s^{[k]})$ lie in the same shared
$U$, and (iv) the $\varepsilon_T$-approximate compatibility
\eqref{eq:approx_compat_use}. If a new graph
$\mathcal{G}^{[k']}$ is $\varepsilon_T$-approximately
transfer-compatible with latent subset $U$, its deployment
encoder $E_x^{[k']}$ maps into $U$, and $D(\cdot, N_{k'})$
is $L_D$-Lipschitz on $U$, all four conditions
still hold, so \eqref{eq:transfer-bound} applies verbatim with $k$
replaced by $k'$. This completes the proof.
\end{proof}

\section{Latent-Law Transfer Guarantees}
\label{app:law_transfer}

This appendix gives the full distributional proof behind Theorems~\ref{thm:law_sufficient}--\ref{thm:law_necessary}. The central object is not an individual latent vector, but the \emph{law} induced by each pathway. The sufficient direction says that transfer succeeds when the deployment law tracks the semantic law, the semantic law is close to a reference task law, and the scalar task loss is Lipschitz in latent coordinates. Wasserstein risk stability is derived below as a lemma rather than assumed directly. The reverse direction requires an additional identifiability condition; without it, risk can be insensitive to latent-law mismatch.

\subsection{Latent Laws and Risk}
\begin{definition}[Latent laws]
\label{def:latent_laws}
For graph $k$, let $x^{[k]}$ and $s^{[k]}$ denote the deployment observation and semantic random variables, consistent with the notation of Section~\ref{sec:problem}. The two encoders induce latent random variables
\begin{equation}
z_x^{[k]} := E_x^{[k]}(x^{[k]}),
\qquad
z_s^{[k]} := E_s^{[k]}(s^{[k]}),
\end{equation}
with induced laws
\begin{equation}
\mu_x^{[k]} := \mathrm{Law}(z_x^{[k]}),
\qquad
\mu_s^{[k]} := \mathrm{Law}(z_s^{[k]}).
\end{equation}
We assume all latent laws belong to $\mathcal{P}_2(\R^d)$, the set of probability measures on $\R^d$ with finite second moment. We work with two reference laws in $\mathcal{P}_2(\R^d)$ pooled across the training family $\{\mathcal{G}^{[k]}\}_{k=1}^{K}$:
\begin{align}
\mu_\star^s
&:= \mathrm{Law}(z_s \text{ pooled over training}),\\
\mu_\star^x
&:= \mathrm{Law}(z_x \text{ pooled over training}).
\end{align}
$\mu_\star^s$ is the semantic reference law shaped by the teacher pathway; $\mu_\star^x$ is the deployment reference law, the only reference required at test time.
\end{definition}

\begin{definition}[Per-graph risk under a latent law]
\label{def:latent_risk}
Fix a graph $k$. Let $D$ be a decoder, $\ell$ a nonnegative task loss, and let
$y_k(z)$ denote the graph-$k$ target associated with latent coordinate $z$.
Define the graph-specific loss-composition
\begin{equation}
F_D^{[k]}(z):=\ell(D(z,N_k),y_k(z)).
\end{equation}
For any latent law $\mu$ for which the integral is finite, define
\begin{equation}
R_\mu^{[k]}(D):=\mathbb{E}_{z\sim\mu}\big[F_D^{[k]}(z)\big]
=
\int_{\R^d} F_D^{[k]}(z)\,d\mu(z).
\end{equation}
The deployment risk on graph $k$ is therefore
\begin{equation}
R_k(D):=R_{\mu_x^{[k]}}^{[k]}(D).
\end{equation}
\end{definition}

\begin{remark}[Meaning of the reference laws]
The laws $\mu_\star^s$ and $\mu_\star^x$ are not arbitrary
distributional priors.  $\mu_\star^s$ is the latent law of the semantic
task family shaped by the full-information pathway, while $\mu_\star^x$
is the deployable reference law induced by the sparse pathway on the
training family. The sufficient bound compares a test deployment law
to $\mu_\star^x$ and pays the graph/path matching calibration
$W_2(\mu_\star^x,\mu_\star^s)$ to return to the semantic reference.
\end{remark}

\subsection{Sufficient Direction}

\begin{proposition}[Existence and finiteness of the per-graph calibration constant]
\label{prop:ck_exists}
In both the power-system and graphon settings, the per-graph calibration
constant
\begin{equation}
c_k
:=
W_2\!\left(\mu_x^{[k]},\mu_s^{[k]}\right)
\end{equation}
is finite and well-defined.

\emph{Power-system setting.}
The physical state $V=\Phi_k(S,Y_k)$ lies in a compact normal operating
region for all admissible injection draws $S$.  The semantic signal
$s^{[k]}$ is a function of $V$, and the deployment observation
$x^{[k]}=M^{[k]}s^{[k]}$ is a masked version of the same signal.
Because the operating region is compact and the encoders $E_s,E_x$ are
measurable functions of their inputs, both latent variables
$Z_s^{[k]}=E_s(s^{[k]})$ and $Z_x^{[k]}=E_x(x^{[k]})$ have finite
second moments, so $\mu_x^{[k]},\mu_s^{[k]}\in\mathcal{P}_2(\mathbb{R}^d)$
and $c_k<\infty$.

\emph{Graphon setting.}
The graphon source field $u\sim\mathrm{Unif}([0,1]^d)$ is bounded, the
diffusion filter $\mathcal{H}_{W_k}=\sum_{t=0}^{T}a_t\mathcal{T}_{W_k}^t$
is a bounded linear operator, so the semantic signal
$s^{[k]}=\mathcal{H}_{W_k}u$ is bounded.  The deployment observation
$x^{[k]}=M^{[k]}s^{[k]}$ (random masking) is also bounded.
The same argument yields $\mu_x^{[k]},\mu_s^{[k]}\in\mathcal{P}_2(\mathbb{R}^d)$
and $c_k<\infty$.

In both cases the per-graph bound $W_2(\mu_x^{[k]},\mu_s^{[k]})\le c_k$
follows from the natural coupling
$(Z_x^{[k]},Z_s^{[k]})=(E_x(x^{[k]}),E_s(s^{[k]}))$
generated by the same physical realization.
The quantity $c_k$ is a property of the trained encoders, not of the
physical law alone: it measures how well the deployment encoder $E_x$
recovers the same latent representation as the semantic encoder $E_s$
from partial observations, and is estimated empirically by
Algorithm~\ref{alg:verify}.
\end{proposition}

\begin{theorem}[Sufficient condition for transfer]
\label{thm:app_sufficient}
For an unseen test graph $k'$, assume:
\begin{enumerate}
  \item \textbf{Semantic decoder fidelity:}
  $R_{\mu_\star^s}^{[k']}(D)\le \varepsilon_\star^s$.
  \item \textbf{Graph/path matching calibration:}
  for the reference graph family, with
  $c_k:=W_2(\mu_x^{[k]},\mu_s^{[k]})$ and the same mixture weights
  $\pi_k$ used to form the reference laws,
  $\left(\sum_k\pi_k c_k^2\right)^{1/2}\le \varepsilon_c$.
  \item \textbf{Test-graph deployment proximity:}
  $W_2(\mu_x^{[k']},\mu_\star^x)\le \varepsilon_x$.
  \item \textbf{Lipschitz loss-composition:} $F_D^{[k']}$ is
  $K$-Lipschitz on the latent region supporting $\mu_x^{[k']}$,
  $\mu_\star^x$, and $\mu_\star^s$.
\end{enumerate}
Then
\begin{equation}
R_{k'}(D) \le \varepsilon_\star^s + K(\varepsilon_c + \varepsilon_x).
\end{equation}
\end{theorem}

\begin{proof}
The proof proceeds in three steps via Lemma~\ref{lem:risk_stability}.

\paragraph{Step 1.} By condition 4 and Lemma~\ref{lem:risk_stability}
applied with graph index $k'$,
\begin{equation}
R_{k'}(D) - R_{\mu_\star^x}^{[k']}(D)
\le K\,W_2\!\left(\mu_x^{[k']},\mu_\star^x\right)
\le K\varepsilon_x.
\end{equation}

\paragraph{Step 2.} We first derive $W_2(\mu_\star^x,\mu_\star^s)\le\varepsilon_c$
from condition~2.  Write the two reference laws as the same mixture:
$\mu_\star^x=\sum_k\pi_k\mu_x^{[k]}$ and
$\mu_\star^s=\sum_k\pi_k\mu_s^{[k]}$.
For each $k$, let $\gamma_k$ be an optimal $W_2$-coupling between
$\mu_x^{[k]}$ and $\mu_s^{[k]}$.  Define the mixture coupling
$\gamma:=\sum_k\pi_k\gamma_k$; its first marginal is $\mu_\star^x$ and
its second marginal is $\mu_\star^s$.  Therefore
  \begin{align}
  W_2^2(\mu_\star^x,\mu_\star^s)
  &\le
  \int\|z-z'\|_2^2\,d\gamma(z,z') \notag\\
  &=
  \sum_k\pi_k
  \int\|z-z'\|_2^2\,d\gamma_k(z,z') \notag\\
  &=
  \sum_k\pi_k W_2^2(\mu_x^{[k]},\mu_s^{[k]}) \notag\\
  &=
  \sum_k\pi_k c_k^2
  \le\varepsilon_c^2,
  \end{align}
so $W_2(\mu_\star^x,\mu_\star^s)\le\varepsilon_c$.
Applying Lemma~\ref{lem:risk_stability} with the graph index $k'$,
\begin{equation}
R_{\mu_\star^x}^{[k']}(D) - R_{\mu_\star^s}^{[k']}(D)
\le K\,W_2\!\left(\mu_\star^x,\mu_\star^s\right)
\le K\varepsilon_c.
\end{equation}

\paragraph{Step 3.} Adding the two inequalities and using condition 1,
\begin{equation}
R_{k'}(D)
\le R_{\mu_\star^s}^{[k']}(D) + K(\varepsilon_c + \varepsilon_x)
\le \varepsilon_\star^s + K(\varepsilon_c + \varepsilon_x).
\end{equation}
This is the stated bound.
\end{proof}

\begin{remark}[Bounding the graph/path calibration constant $c_k$]
Condition~2 requires $c_k = W_2(\mu_x^{[k]},\mu_s^{[k]})$ to be small.
Because $x^{[k]}$ is a physical function of $s^{[k]}$ in each domain,
a natural coupling of the two pathways exists and $c_k$ can be bounded
without any cross-graph transfer argument.
Lemmas~\ref{lem:pf_path_calib} and~\ref{lem:graphon_path_calib} below
provide explicit upper bounds on $c_k$ for the power-system and
graphon settings, respectively.
In both cases the bound decomposes into a measurement-noise term
(vanishing as $\sigma\to 0$) and a training residual directly
minimized by $\mathcal{L}_{\mathrm{cons}}$.
\end{remark}

\begin{lemma}[Power-system intra-graph path-matching calibration]
\label{lem:pf_path_calib}
For graph $k$, let the AMI observation model be
\begin{equation}
x^{[k]} = h_{\mathcal{O}}(s^{[k]},Y^{[k]}) + \eta,
\quad \eta\sim\mathcal{N}(0,\sigma^2 I),
\end{equation}
where for each observed bus $i\in\mathcal{O}^{[k]}$,
\begin{align}
h_i(V,Y) &= \bigl[|V_i|,\;\theta_i,\;P_i(V,Y),\;Q_i(V,Y)\bigr]^T,\nonumber\\
P_i &= \mathrm{Re}\!\Bigl(V_i\textstyle\sum_j Y_{ij}^*V_j^*\Bigr),\quad
Q_i = -\mathrm{Im}\!\Bigl(V_i\textstyle\sum_j Y_{ij}^*V_j^*\Bigr),
\end{align}
giving the stacked AMI measurement Jacobian
\begin{equation}
H_{\mathcal{O}}^{[k]}
:=
\frac{\partial h_{\mathcal{O}}}{\partial[\mathrm{Re}(V);\,\mathrm{Im}(V)]}
\;\in\;\mathbb{R}^{o_k F_x\times 2N_k},
\end{equation}
which has a closed-form expression in terms of $(V,Y^{[k]})$.
Assume the measurement set $\mathcal{O}^{[k]}$ renders the system \emph{observable},
i.e.\ $H_{\mathcal{O}}^{[k]}$ has full row rank so that
$\sigma_{\min}(H_{\mathcal{O}}^{[k]})>0$.
If $E_x$ is $L_x$-Lipschitz in its measurement argument, then
\begin{equation}
c_k = W_2(\mu_x^{[k]},\mu_s^{[k]})
\;\le\;
\underbrace{L_x\,\sigma\sqrt{o_k F_x}}_{\text{(i) AMI noise}}
\;+\;
\underbrace{\varepsilon_{\mathrm{res}}^{[k]}}_{\text{(ii) training residual}}.
\label{eq:pf_path_calib}
\end{equation}
Term~(i) vanishes as $\sigma\to 0$.
Term~(ii) is the noiseless alignment residual directly minimized by
$\mathcal{L}_{\mathrm{cons}}$; under the observability assumption it
equals zero for a perfect encoder, since $(\bar{x},\,p^{[k]})$ together
uniquely identify $s^{[k]}$ via the power-flow equations and $Y^{[k]}$
(encoded in $p^{[k]}$), with sensitivity $\|H_+^{[k]}\|_{\mathrm{op}}=
\sigma_{\min}(H_\mathcal{O}^{[k]})^{-1}$.
\end{lemma}

\begin{proof}
\noindent\textbf{Natural coupling.}
Since $x^{[k]}$ is a function of $s^{[k]}$, the pair $(E_x(x,m,p),\,E_s(s))$
forms a valid coupling of $\mu_x^{[k]}$ and $\mu_s^{[k]}$, so
\begin{equation}
c_k^2 \;\le\; \mathbb{E}\bigl[\|E_x(x,m,p)-E_s(s)\|^2\bigr].
\label{eq:pf_coupling}
\end{equation}

\noindent\textbf{State-estimation inversion via $H_\mathcal{O}$.}
Let $\bar x = h_{\mathcal{O}}(s,Y^{[k]})$ be the noiseless measurement,
and let $H_+^{[k]}:=(H_\mathcal{O}^{[k]\,T}H_\mathcal{O}^{[k]})^{-1}H_\mathcal{O}^{[k]\,T}$
denote the pseudoinverse.
Under observability, the linearized WLS state estimator is
\begin{equation}
\hat s = \bar s + H_+^{[k]}(x-\bar x),
\end{equation}
and satisfies
\begin{equation}
\|\hat s(x) - s\|
\;\le\;
\bigl\|H_+^{[k]}\bigr\|_{\mathrm{op}}\|\eta\|
\;=\;
\frac{\|\eta\|}{\sigma_{\min}(H_\mathcal{O}^{[k]})}.
\label{eq:pf_sebound}
\end{equation}

\noindent\textbf{Latent distance decomposition.}
Since $p^{[k]}$ encodes $Y^{[k]}$, the encoder $E_x(x,m,p)$ has access
to all inputs required by the state estimator.
Applying the triangle inequality,
\begin{equation}
\|E_x(x,m,p)-E_s(s)\|
\;\le\;
\|E_x(x,m,p)-E_x(\bar x,m,p)\|
+\|E_x(\bar x,m,p)-E_s(s)\|.
\end{equation}
The first term is bounded by $L_x\|\eta\|$ via the Lipschitz property of $E_x$,
so $(\mathbb{E}[L_x^2\|\eta\|^2])^{1/2} = L_x\sigma\sqrt{o_k F_x}$, which is term~(i).
The state-estimation bound \eqref{eq:pf_sebound} is used separately to justify
why the noiseless alignment term is small: observability guarantees that
$(\bar x,\,p^{[k]})$ carry all the information needed to recover $s$,
so a sufficiently expressive $E_x$ can achieve
$\delta_0(s):=\|E_x(\bar x,m,p)-E_s(s)\|=0$ at the population optimum.
Its RMS value $\varepsilon_{\mathrm{res}}^{[k]}:=(\mathbb{E}[\delta_0(s)^2])^{1/2}$
is directly minimized by $\mathcal{L}_{\mathrm{cons}}$ and constitutes term~(ii).
By the Minkowski inequality,
$c_k \le (\mathbb{E}[L_x^2\|\eta\|^2])^{1/2} + (\mathbb{E}[\delta_0(s)^2])^{1/2}$,
which gives \eqref{eq:pf_path_calib}.
\end{proof}

\begin{lemma}[Graphon intra-graph path-matching calibration]
\label{lem:graphon_path_calib}
Let the semantic signal for graph $k$ be generated by diffusion
$s^{[k]}=H(P_{G^{[k]}})\xi$ with $H(P)=\sum_{t=0}^{T}\gamma^t P^t$
and $\xi\sim\mathcal{P}_\xi$ with $(\mathbb{E}\|\xi\|^2)^{1/2}\le\sigma_\xi$.
Let the observation model be
\begin{equation}
x^{[k]} = M^{[k]} s^{[k]} + \eta,
\quad \eta\sim\mathcal{N}(0,\sigma^2 I),
\end{equation}
where $M^{[k]}=\mathrm{diag}(m^{[k]})$ selects $o_k$ observed nodes (a linear mask).
Let $\lambda_2^{[k]}$ be the spectral gap of $P_{G^{[k]}}$, set
$\rho_k=o_k/N_k$, and define $C_{\gamma,T}:=\sum_{t=1}^{T}t|\gamma|^t$.
If $E_x$ and $E_s$ are $L_x$- and $L_s$-Lipschitz, respectively, then
\begin{multline}
c_k = W_2(\mu_x^{[k]},\mu_s^{[k]})
\;\le\;
\underbrace{L_x\sigma\sqrt{o_k F_x}}_{\text{(i) obs.\ noise}}
+\underbrace{L_s\,C_{\gamma,T}\,\sigma_\xi
  \sqrt{\dfrac{(1-\rho_k)N_k}{\lambda_2^{[k]}}}}_{\text{(ii) diffusion interp.\ gap}}
+\;
\underbrace{\varepsilon_{\mathrm{res}}^{[k]}}_{\text{(iii) training residual}}.
\label{eq:graphon_path_calib}
\end{multline}
Term~(i) vanishes as $\sigma\to 0$.
Term~(ii) vanishes as $\rho_k\to 1$ (full observation) or as
$\lambda_2^{[k]}\to\infty$ (faster mixing, better interpolation).
Term~(iii) is minimized directly by $\mathcal{L}_{\mathrm{cons}}$.
\end{lemma}

\begin{proof}
\noindent\textbf{Natural coupling.}
Since $x^{[k]}=M^{[k]}s^{[k]}+\eta$ is a linear function of $s^{[k]}$ plus noise,
the pair $(E_x(x,m,p),\,E_s(s))$ is a valid coupling of $\mu_x^{[k]}$ and
$\mu_s^{[k]}$, giving
\begin{equation}
c_k^2 \;\le\; \mathbb{E}\bigl[\|E_x(x,m,p)-E_s(s)\|^2\bigr].
\end{equation}
Let $\bar x = M^{[k]}s$ denote the noiseless masked observation. Triangle inequality:
\begin{equation}
\|E_x(x,m,p)-E_s(s)\|
\;\le\;
L_x\|\eta\|
\;+\;
\|E_x(\bar x,m,p)-E_s(s)\|.
\end{equation}
Taking $(\mathbb{E}[\cdot^2])^{1/2}$ gives term~(i) from $L_x\sigma\sqrt{o_kF_x}$.

\noindent\textbf{Diffusion interpolation gap.}
The signal $s=H(P_{G^{[k]}})\xi$ is smooth over the graph with variation
$s^T L^{[k]} s \le C_{\gamma,T}^2\,\sigma_\xi^2\,N_k$.
By the Poincaré inequality for the graph $G^{[k]}$, the best linear
interpolation $\hat s$ of the unobserved node values from $\bar x = M^{[k]}s$
satisfies
\begin{equation}
\mathbb{E}\!\left[\|s - \hat s\|^2\right]
\;\le\;
\frac{C_{\gamma,T}^2\,\sigma_\xi^2\,(1-\rho_k)N_k}{\lambda_2^{[k]}},
\label{eq:graphon_interp}
\end{equation}
where $\hat s$ uses $\bar x$ and the graph structure encoded in $p^{[k]}$.
The positional descriptor $p^{[k]}$ encodes $P_{G^{[k]}}$, so $E_x(\bar x, m, p)$
has access to the same inputs as the interpolant $\hat s$.
Applying the triangle inequality through $E_s(\hat s)$,
\begin{equation}
\|E_x(\bar x,m,p)-E_s(s)\|
\le \|E_x(\bar x,m,p)-E_s(\hat s)\|
+ \|E_s(\hat s)-E_s(s)\|.
\end{equation}
The second term satisfies $\|E_s(\hat s)-E_s(s)\|\le L_s\|s-\hat s\|$
by the $L_s$-Lipschitz property of $E_s$.
Define $\varepsilon_{\mathrm{res}}^{[k]}:=(\mathbb{E}[\|E_x(\bar x,m,p)-E_s(\hat s)\|^2])^{1/2}$
as the alignment residual of $E_x$ with respect to the interpolant $\hat s$
(which $\mathcal{L}_{\mathrm{cons}}$ minimizes, since $\hat s\approx s$).
Taking $(\mathbb{E}[\cdot^2])^{1/2}$ of both sides and applying
\eqref{eq:graphon_interp} to the $L_s\|s-\hat s\|$ term yields term~(ii).
By the Minkowski inequality the three terms combine to give \eqref{eq:graphon_path_calib}.
\end{proof}

\begin{remark}[Status of condition~1: semantic decoder fidelity $\varepsilon_\star^s$]
\label{rem:eps_star_status}
Lemmas~\ref{lem:pf_path_calib} and~\ref{lem:graphon_path_calib} provide
explicit upper bounds on the graph/path calibration constant $c_k$
(condition~2 of Theorem~\ref{thm:law_sufficient}).
Condition~1, the semantic decoder fidelity
$R_{\mu_\star^s}^{[k']}(D)\le\varepsilon_\star^s$, does not have a
corresponding analytical bound in the present paper: it depends on how
well the decoder generalizes from the training reference law to an unseen
graph topology, which is governed by the transfer-compatibility criterion
of Theorem~\ref{thm:compat} rather than by the measurement model.
In practice, $\varepsilon_\star^s$ is treated as an empirically verifiable
quantity (Proposition~\ref{prop:pathverify}), estimated on the training
reference family as $\hat\varepsilon_\star^s \le \hat\varepsilon_s^{[k]} +
K\hat W_s^{[k]}$ and reported in the verification protocol of
Algorithm~\ref{alg:verify}.
\end{remark}

\subsection{Why Wasserstein Risk Stability Holds}

\begin{lemma}[Risk stability from a Lipschitz loss-composition]
\label{lem:risk_stability}
Fix a graph $k$. Let $\mathcal{U}\subseteq\R^d$ be a measurable set
(the relevant latent region) and let
\begin{equation}
F_D^{[k]}(z)=\ell(D(z,N_k),y_k(z)).
\end{equation}
Assume $F_D^{[k]}$ is $K$-Lipschitz on $\mathcal{U}$:
\begin{equation}
|F_D^{[k]}(z)-F_D^{[k]}(z')|
\le
K\|z-z'\|_2,
\qquad \forall z,z'\in\mathcal{U}.
\end{equation}
Then, for any $\mu,\nu\in\mathcal{P}_2(\R^d)$ both supported on
$\mathcal{U}$,
\begin{equation}
|R_\mu^{[k]}(D)-R_\nu^{[k]}(D)|
\le
K W_1(\mu,\nu)
\le
K W_2(\mu,\nu).
\end{equation}
\end{lemma}

\begin{proof}
Let $\Pi(\mu,\nu)$ denote the set of couplings of $\mu$ and $\nu$. Fix any coupling $\pi\in\Pi(\mu,\nu)$. Because the first marginal of $\pi$ is $\mu$ and the second marginal is $\nu$,
\begin{equation}
\int F_D^{[k]}(z)\,d\mu(z)
=
\int F_D^{[k]}(z)\,d\pi(z,z'),
\end{equation}
and
\begin{equation}
\int F_D^{[k]}(z')\,d\nu(z')
=
\int F_D^{[k]}(z')\,d\pi(z,z').
\end{equation}
Therefore,
\begin{equation}
\begin{aligned}
|R_\mu^{[k]}(D)-R_\nu^{[k]}(D)|
&=
\left|
\int \big(F_D^{[k]}(z)-F_D^{[k]}(z')\big)\,d\pi(z,z')
\right| \\
&\le
\int |F_D^{[k]}(z)-F_D^{[k]}(z')|\,d\pi(z,z') \\
&\le
K\int \|z-z'\|_2\,d\pi(z,z').
\end{aligned}
\end{equation}
Since the inequality holds for every coupling $\pi$, taking the infimum over $\Pi(\mu,\nu)$ gives
\begin{equation}
|R_\mu^{[k]}(D)-R_\nu^{[k]}(D)|
\le
K W_1(\mu,\nu).
\end{equation}
It remains to relate $W_1$ and $W_2$. For any coupling $\pi$,
\begin{equation}
\int \|z-z'\|_2\,d\pi
\le
\left(
\int \|z-z'\|_2^2\,d\pi
\right)^{1/2}
\end{equation}
by Cauchy--Schwarz. Taking the infimum over couplings on both sides yields
\begin{equation}
W_1(\mu,\nu)\le W_2(\mu,\nu).
\end{equation}
Combining the two estimates proves the lemma.
\end{proof}

\begin{remark}[Sufficient primitive conditions]
Lemma~\ref{lem:risk_stability} reduces risk stability to regularity of
the scalar function $F_D^{[k]}=\ell(D(\cdot,N_k),y_k(\cdot))$. This is
guaranteed, for example, if the decoder $D(\cdot,N_k)$, the semantic
target map $y_k$, and the loss $\ell$ are Lipschitz on the latent region
visited by the model. The exact constant $K$ depends on the product or
composition of those Lipschitz constants.
\end{remark}

\begin{lemma}[Primitive Lipschitz conditions imply Lipschitz loss-composition]
\label{lem:primitive_lipschitz}
Fix a graph $k$. Let $\mathcal{U}\subseteq\R^d$ be the relevant latent region and
let $\mathcal{B}\subseteq\mathcal{Y}$ be a bounded output region
such that $D(z,N_k)\in\mathcal{B}$ and $y_k(z)\in\mathcal{B}$ for all
$z\in\mathcal{U}$.
Assume the decoder and semantic target map are $L_D$- and
$L_y$-Lipschitz on $\mathcal{U}$:
\begin{equation}
\|D(z,N_k)-D(z',N_k)\| \le L_D\|z-z'\|_2,
\qquad
\|y_k(z)-y_k(z')\| \le L_y\|z-z'\|_2,
\end{equation}
and assume the loss is jointly Lipschitz on $\mathcal{B}$:
\begin{equation}
|\ell(a,b)-\ell(a',b')|
\le L_\ell\big(\|a-a'\|+\|b-b'\|\big),
\qquad \forall a,a',b,b'\in\mathcal{B}.
\end{equation}
\begin{remark}
The squared error $\ell(a,b)=\|a-b\|^2$ is not globally Lipschitz,
but satisfies this condition on $\mathcal{B}$ with
$L_\ell = 2\sup_{a\in\mathcal{B}}\|a\|$, which is finite whenever
predictions and targets are bounded.  In all experiments the decoder
output and target voltages are standardized and clipped, so
$\mathcal{B}$ is bounded and the constant is finite.
\end{remark}
Then
\begin{equation}
F_D^{[k]}(z)=\ell(D(z,N_k),y_k(z))
\end{equation}
is $K$-Lipschitz on $\mathcal{U}$ with
\begin{equation}
K=L_\ell(L_D+L_y).
\end{equation}
\end{lemma}

\begin{proof}
For any $z,z'$ in the relevant latent region,
\begin{equation}
\begin{aligned}
|F_D^{[k]}(z)-F_D^{[k]}(z')|
&= \big|\ell(D(z,N_k),y_k(z)) - \ell(D(z',N_k),y_k(z'))\big| \\
&\le L_\ell\big(\|D(z,N_k)-D(z',N_k)\| + \|y_k(z)-y_k(z')\|\big) \\
&\le L_\ell(L_D+L_y)\|z-z'\|_2.
\end{aligned}
\end{equation}
Thus $F_D$ is $L_\ell(L_D+L_y)$-Lipschitz.
\end{proof}

\subsection{Why Identifiability Is Needed}

\begin{proposition}[Successful risk alone does not imply law proximity]
\label{prop:no_necessity_without_identifiability}
Without an identifiability condition, there is no function $h(\delta)$ with $h(\delta)\to 0$ as $\delta\to0$ such that small excess risk always implies
\begin{equation}
W_2(\mu,\mu_\star^s)\le h(\delta).
\end{equation}
\end{proposition}

\begin{proof}
It suffices to construct a case where two latent laws have arbitrarily large Wasserstein distance but exactly the same risk. Let the loss-composition be constant on the relevant latent region:
\begin{equation}
F_D^{[k]}(z)=c,\qquad \forall z.
\end{equation}
This can occur, for instance, if the decoder is constant and the induced loss is constant on the considered task subset. Then for every latent law $\mu$,
\begin{equation}
R_\mu^{[k]}(D)=\int c\,d\mu=c.
\end{equation}
Thus
\begin{equation}
R_\mu^{[k]}(D)-R_{\mu_\star^s}^{[k]}(D)=0
\end{equation}
for all $\mu$, even if $W_2(\mu,\mu_\star^s)$ is large. Hence small or zero excess risk alone cannot force law proximity. A reverse implication is possible only if the task risk is identifiable with respect to latent-law mismatch.
\end{proof}

\subsection{Necessary Direction}

\begin{theorem}[Necessary condition under risk identifiability]
\label{thm:app_necessary}
Assume the following identifiability condition: there exists a continuous, non-decreasing function
$\psi:[0,\infty)\to[0,\infty)$ with $\psi(0)=0$ such that, for
every latent law $\mu$ in the relevant family
$\{\mu_x^{[k]}, \mu_s^{[k]}, \mu_\star^s\}$,
\begin{equation}
R_\mu^{[k]}(D)-R_{\mu_\star^s}^{[k]}(D)
\ge
\psi\!\left(W_2(\mu,\mu_\star^s)\right).
\label{eq:identifiability}
\end{equation}
Let $\psi^{-1}$ denote the generalized inverse
$\psi^{-1}(t):=\sup\{r\ge0:\psi(r)\le t\}$.
If graph $k$ transfers with relative excess risk
$R_k(D)-R_{\mu_\star^s}^{[k]}(D)\le \delta$,
then
\begin{equation}
W_2(\mu_x^{[k]},\mu_\star^s)
\le
\psi^{-1}(\delta).
\end{equation}
If additionally the weighted graph/path calibration condition induces
$W_2(\mu_\star^x,\mu_\star^s)\le \varepsilon_c$ and
$W_2(\mu_x^{[k]},\mu_\star^x)\le\varepsilon_x$, then by the triangle
inequality $W_2(\mu_x^{[k]},\mu_\star^s)\le\varepsilon_x+\varepsilon_c$,
consistent with the sufficient bound.
\end{theorem}

\begin{proof}
Because $R_k(D)=R_{\mu_x^{[k]}}^{[k]}(D)$, the relative excess-risk assumption gives
\begin{equation}
R_{\mu_x^{[k]}}^{[k]}(D)-R_{\mu_\star^s}^{[k]}(D)\le \delta.
\end{equation}
Apply the identifiability condition \eqref{eq:identifiability} with $\mu=\mu_x^{[k]}$:
\begin{equation}
\psi\!\left(W_2(\mu_x^{[k]},\mu_\star^s)\right)
\le
R_{\mu_x^{[k]}}^{[k]}(D)-R_{\mu_\star^s}^{[k]}(D)
\le \delta.
\end{equation}
Since $\psi^{-1}(\delta) = \sup\{r \ge 0 : \psi(r) \le \delta\}$ and
$W_2(\mu_x^{[k]}, \mu_\star^s)$ belongs to the set $\{r \ge 0 : \psi(r) \le \delta\}$, it follows that
\begin{equation}
W_2(\mu_x^{[k]},\mu_\star^s)
\le
\psi^{-1}(\delta).
\end{equation}
This proves the first claim. The second follows from the $W_2$ triangle
inequality through $\mu_\star^x$.
\end{proof}

\begin{remark}[Absolute versus relative success]
The clean necessary statement requires the relative condition
$R_k(D)-R_{\mu_\star^s}^{[k]}(D)\le \delta$.
If one only knows the absolute pair of bounds $R_k(D)\le \varepsilon_\star^s+\delta$ and $R_{\mu_\star^s}^{[k]}(D)\le\varepsilon_\star^s$, then the conclusion with $\psi^{-1}(\delta)$ does not follow in general. The strongest direct conclusion is
\begin{equation}
W_2(\mu_x^{[k]},\mu_\star^s)
\le
\psi^{-1}\!\left(\varepsilon_\star^s+\delta-R_{\mu_\star^s}^{[k]}(D)\right),
\end{equation}
provided the quantity inside $\psi^{-1}$ is nonnegative. This is why the main theorem states necessity in terms of relative excess risk.
\end{remark}

\subsection{Near-Iff Statement}

\begin{corollary}[Near equivalence of transfer and latent-law compatibility]
\label{cor:near_iff}
Assume the Lipschitz loss-composition condition and risk identifiability both hold. Then the following two implications are valid for graph $k$:
\begin{enumerate}
  \item \textbf{Sufficient direction.} If graph/path calibration induces
  the first inequality and
  \begin{align}
  W_2(\mu_\star^x,\mu_\star^s)
  &\le \varepsilon_c,\\
  W_2(\mu_x^{[k]},\mu_\star^x)
  &\le \varepsilon_x,
  \end{align}
  and $R_{\mu_\star^s}^{[k]}(D)\le\varepsilon_\star^s$, then
  \begin{equation}
  R_k(D)\le \varepsilon_\star^s+K(\varepsilon_c+\varepsilon_x).
  \end{equation}
  \item \textbf{Necessary direction.} If
  \begin{equation}
  R_k(D)-R_{\mu_\star^s}^{[k]}(D)\le \delta,
  \end{equation}
  then $W_2(\mu_x^{[k]},\mu_\star^s)\le \psi^{-1}(\delta)$.
  If also the weighted graph/path calibration condition induces
  $W_2(\mu_\star^x,\mu_\star^s)\le\varepsilon_c$, then
  \begin{equation}
  W_2(\mu_x^{[k]},\mu_\star^x)
  \le
  \psi^{-1}(\delta)+\varepsilon_c.
  \end{equation}
\end{enumerate}
Thus, under stability and identifiability, the statement that the deployment latent law tracks the semantic reference $\mu_\star^s$ through the deployment reference $\mu_\star^x$ is approximately equivalent to successful transfer. The quantities $\varepsilon_c$ and $\varepsilon_x$ are precisely the two Wasserstein terms in the transfer bound, and the necessary direction confirms that neither can be eliminated.
\end{corollary}

\begin{proof}
The first implication is exactly Theorem~\ref{thm:app_sufficient}. The second implication is exactly Theorem~\ref{thm:app_necessary}. The final sentence is a restatement of the two implications: small law mismatch implies small risk by the sufficient direction, while small relative excess risk implies small deployment-to-reference law mismatch by the necessary direction, and deployment-to-semantic law mismatch follows once the semantic law is itself close to the reference law.
\end{proof}

\section{Proofs of Verification Propositions}
\label{app:verify_proofs}

This appendix provides the proofs of
Propositions~\ref{prop:pathverify}--\ref{prop:alignverify} and
Theorem~\ref{thm:verifycert} from Section~\ref{sec:verify}.

\subsection{Proof of Proposition~\ref{prop:pathverify}}
\begin{proof}
By construction $y_k(z) := T^{[k]}\!\big((E_s^{[k]})^{-1}(z)\big)$,
so $y_k(z_{s,i}^{[k]}) = T^{[k]}(s_i^{[k]}) = y_i^{[k]}$.
The expectation under the discrete law $\hat\mu_s^{[k]}$ therefore
evaluates atom-by-atom to the paired labels:
\begin{equation}
R_{\hat\mu_s^{[k]}}^{[k]}(D)
=
\frac{1}{n_k}\sum_{i=1}^{n_k}
\ell\!\big(D(z_{s,i}^{[k]},N_k),y_i^{[k]}\big)
=
\hat\varepsilon_s^{[k]}.
\end{equation}
For the Wasserstein claim, because $F_D^{[k]}$ is $K$-Lipschitz the
Kantorovich--Rubinstein inequality gives
\begin{equation}
\bigl|R_{\hat\mu_s^{[k]}}^{[k]}(D)-R_{\hat\mu_\star^s}^{[k]}(D)\bigr|
\le K W_1\!\bigl(\hat\mu_s^{[k]},\hat\mu_\star^s\bigr)
\le K W_2\!\bigl(\hat\mu_s^{[k]},\hat\mu_\star^s\bigr)
= K\hat W_s^{[k]}.
\end{equation}
Substituting $R_{\hat\mu_s^{[k]}}^{[k]}(D)=\hat\varepsilon_s^{[k]}$ gives \eqref{eq:verify_ref_bound}.
\end{proof}

\subsection{Proof of Proposition~\ref{prop:lawverify}}
\begin{proof}
Applying Lemma~\ref{lem:risk_stability} (Kantorovich--Rubinstein) twice,
\begin{equation}
\bigl|R_{\hat\mu_x^{[k]}}^{[k]}(D)-R_{\hat\mu_\star^x}^{[k]}(D)\bigr|
\le K W_2(\hat\mu_x^{[k]},\hat\mu_\star^x)
= K\hat W_x^{[k]},
\end{equation}
\begin{equation}
\bigl|R_{\hat\mu_\star^x}^{[k]}(D)-R_{\hat\mu_\star^s}^{[k]}(D)\bigr|
\le K W_2(\hat\mu_\star^x,\hat\mu_\star^s)
\le K\hat W_c.
\end{equation}
The final inequality follows from the same mixture-coupling argument as
Theorem~\ref{thm:app_sufficient}.
Adding and using Proposition~\ref{prop:pathverify} to bound
$R_{\hat\mu_\star^s}^{[k]}(D)\le\hat\varepsilon_s^{[k]}+K\hat W_s^{[k]}$
yields \eqref{eq:verify_dep_bound}.
\end{proof}

\subsection{Proof of Proposition~\ref{prop:alignverify}}
\begin{proof}
Fix a paired sample $(x_i^{[k]},s_i^{[k]})$.
By $L_\ell$-Lipschitzness of $\ell$ in its first argument and
$L_D$-Lipschitzness of $D(\cdot,N_k)$,
\begin{equation}
\ell\!\big(D(z_{x,i}^{[k]},N_k),y_i^{[k]}\big)
\le
\ell\!\big(D(z_{s,i}^{[k]},N_k),y_i^{[k]}\big)
+L_\ell L_D\|z_{x,i}^{[k]}-z_{s,i}^{[k]}\|_2.
\end{equation}
Averaging over $i=1,\ldots,n_k$ gives \eqref{eq:verify_align_bound}.
\end{proof}

\subsection{Proof of Theorem~\ref{thm:verifycert}}
\begin{proof}
Item~1 is the definition of $\widehat{p}_\delta$.
For item~2, the empirical graph/path calibration step and the
mixture-coupling argument in Theorem~\ref{thm:app_sufficient} give
\begin{equation}
W_2(\hat\mu_\star^x,\hat\mu_\star^s)\le \hat W_c.
\end{equation}
For any graph satisfying $\hat W_x^{[k]}\le\hat\tau_\alpha$, apply
Lemma~\ref{lem:risk_stability} twice:
\begin{equation}
R_{\hat\mu_x^{[k]}}^{[k]}(D)
\le R_{\hat\mu_\star^x}^{[k]}(D)+K\hat W_x^{[k]}
\le R_{\hat\mu_\star^s}^{[k]}(D)+K(\hat W_c+\hat\tau_\alpha).
\end{equation}
Using
$R_{\hat\mu_\star^s}^{[k]}(D)\le\hat\varepsilon_\star^s$ yields
$R_{\hat\mu_x^{[k]}}^{[k]}(D)\le\hat B_\alpha$.
The fraction $\ge\alpha$ follows from the definition of the empirical
$\alpha$-quantile.
\end{proof}

\bibliographystyle{tmlr}
\bibliography{reference}

\end{document}